\documentclass[10pt,conference,letterpaper]{IEEEtran}
\IEEEoverridecommandlockouts

\usepackage{amsmath,amsfonts,bm}

\def\eqref#1{equation~\ref{#1}}

\def\1{\mathbbm{1}}

\DeclareMathAlphabet{\mathsfit}{\encodingdefault}{\sfdefault}{m}{sl}
\SetMathAlphabet{\mathsfit}{bold}{\encodingdefault}{\sfdefault}{bx}{n}

\newcommand{\TV}{D_{\mathrm{TV}}}

\usepackage[colorlinks,linkcolor=blue,citecolor=blue]{hyperref} 
\usepackage{cite}
\usepackage{amsmath,amssymb,amsfonts}
\usepackage{algorithmic}
\usepackage[ruled,linesnumbered,vlined]{algorithm2e}
\usepackage{graphicx}
\usepackage{textcomp}
\usepackage{subfigure}
\usepackage{booktabs}
\usepackage{xcolor}
\usepackage{mathtools}
\usepackage{enumitem}
\usepackage{amsthm}
\usepackage{commath}
\usepackage{amsfonts}
\usepackage{nicefrac}
\usepackage{microtype}
\usepackage[utf8]{inputenc} 
\usepackage[T1]{fontenc} 
\usepackage{bbm}
\usepackage{bm}
\usepackage[capitalize]{cleveref}
\usepackage{nicefrac}
\usepackage{multicol}
\usepackage{float}

\newtheorem{theorem}{Theorem}

\newtheorem{lemma}{Lemma}

\newtheorem{assumption}{Assumption}

\hypersetup{
    colorlinks=true,
    linkcolor=blue,
    filecolor=blue,      
    urlcolor=black,
    citecolor=blue,
}

\def\dy{{P}}
\def\bl{{\mathcal{T}}}
\def\alg{{\texttt{DRPO}}}

\DeclareRobustCommand*{\IEEEauthorrefmark}[1]{\raisebox{0pt}[0pt][0pt]{\textsuperscript{\footnotesize #1}}}

\def\BibTeX{{\rm D\kern-.05em{\sc i\kern-.025em b}\kern-.08em
    T\kern-.1667em\lower.7ex\hbox{E}\kern-.125emX}}

\allowdisplaybreaks[4]

\begin{document}

\title{Federated Offline Policy Optimization with\\Dual Regularization
}

\author{
\IEEEauthorblockN{
Sheng Yue\IEEEauthorrefmark{1}, 
Zerui Qin\IEEEauthorrefmark{1}, 
Xingyuan Hua\IEEEauthorrefmark{2}, 
Yongheng Deng\IEEEauthorrefmark{1}, 
Ju Ren\IEEEauthorrefmark{1}\IEEEauthorrefmark{3}\IEEEauthorrefmark{*}
\thanks{\textsuperscript{*}Corresponding author}}
\IEEEauthorblockA{
\IEEEauthorrefmark{1}Department of Computer Science and Technology, BNRist, Tsinghua University, Beijing, China \\
\IEEEauthorrefmark{2}School of Computer Science and Technology, Beijing Institute of Technology, Beijing, China \\ 
\IEEEauthorrefmark{3}Zhongguancun Laboratory, Beijing, China \\
\href{mailto:shengyue@tsinghua.edu.cn,dyh19@tsinghua.edu.cn,renju@tsinghua.edu.cn}{\texttt{\{shengyue,zeruiqin,dyh19,renju\}@tsinghua.edu.cn}}, \href{mailto:xingyuanhua@bit.edu.cn}{\texttt{xingyuanhua@bit.edu.cn}}
}}

\maketitle

\begin{abstract}

Federated Reinforcement Learning (FRL) has been deemed as a promising solution for intelligent decision-making in the era of Artificial Internet of Things. However, existing FRL approaches often entail repeated interactions with the environment during local updating, which can be prohibitively expensive or even infeasible in many real-world domains. To overcome this challenge, this paper proposes a novel offline federated policy optimization algorithm, named \alg{}, which enables distributed agents to collaboratively learn a decision policy only from private and static data without further environmental interactions. \alg{} leverages dual regularization, incorporating both the local behavioral policy and the global aggregated policy, to judiciously cope with the intrinsic two-tier distributional shifts in offline FRL. Theoretical analysis characterizes the impact of the dual regularization on performance, demonstrating that by achieving the right balance thereof, \alg{} can effectively counteract distributional shifts and ensure strict policy improvement in each federative learning round. Extensive experiments validate the significant performance gains of \alg{} over baseline methods.

\end{abstract}


\section{Introduction}
\label{sec:introduction}



With the rapid proliferation of intelligent decision-making applications and the ongoing advancement of Federated Learning (FL)~\cite{mcmahan2017communication,li2020sample,xu2021edge}, \emph{Federated Reinforcement Learning} (FRL) has garnered a surge of interest recently, whereby distributed agents can join forces to learn a decision policy without the need to share their raw trajectories~\cite{liu2019lifelong,nadiger2019federated,zhuo2020federated,fan2021fault,khodadadian2022federated}. FRL is expected to cope with the data hungry of centralized Reinforcement Learning (RL) while complying with the requirement of privacy preservation or data confidentiality~\cite{khodadadian2022federated}. It has demonstrated remarkable potential across a wide range of real-world systems, including robot navigation~\cite{liu2019lifelong}, resource management in networking~\cite{yu2020deep}, and control of IoT devices~\cite{lim2020federated}.

Albeit with promising applications, the practical deployment of current FRL approaches is hampered by a major challenge: \emph{during local updating, agents require repeated environmental interaction to collect online experience with the latest policy.} However, interacting with real systems can be slow and fragile, which can lead to inefficient policy synchronization and lengthy wall-clock training time. More importantly, conducting such online interaction may not be feasible in many settings due to the high costs or risks involved~\cite{levine2020offline}. 
For instance, in healthcare, an unreliable treatment policy may have harmful effects on patients, and in power systems, a suboptimal energy management policy could result in severe economic losses. To overcome the limitation, this work studies \emph{offline FRL} that allows distributed agents to collaboratively build a decision policy only from (private) static data, such as historical records of human driving or doctors' diagnoses, without further online exploration. This data-driven policy learning fashion aligns better with the existing FL paradigm and is appealing for data reuse and safety-sensitive domains.

\begin{figure}[t]
    \centering
    \includegraphics[width=0.995\linewidth]{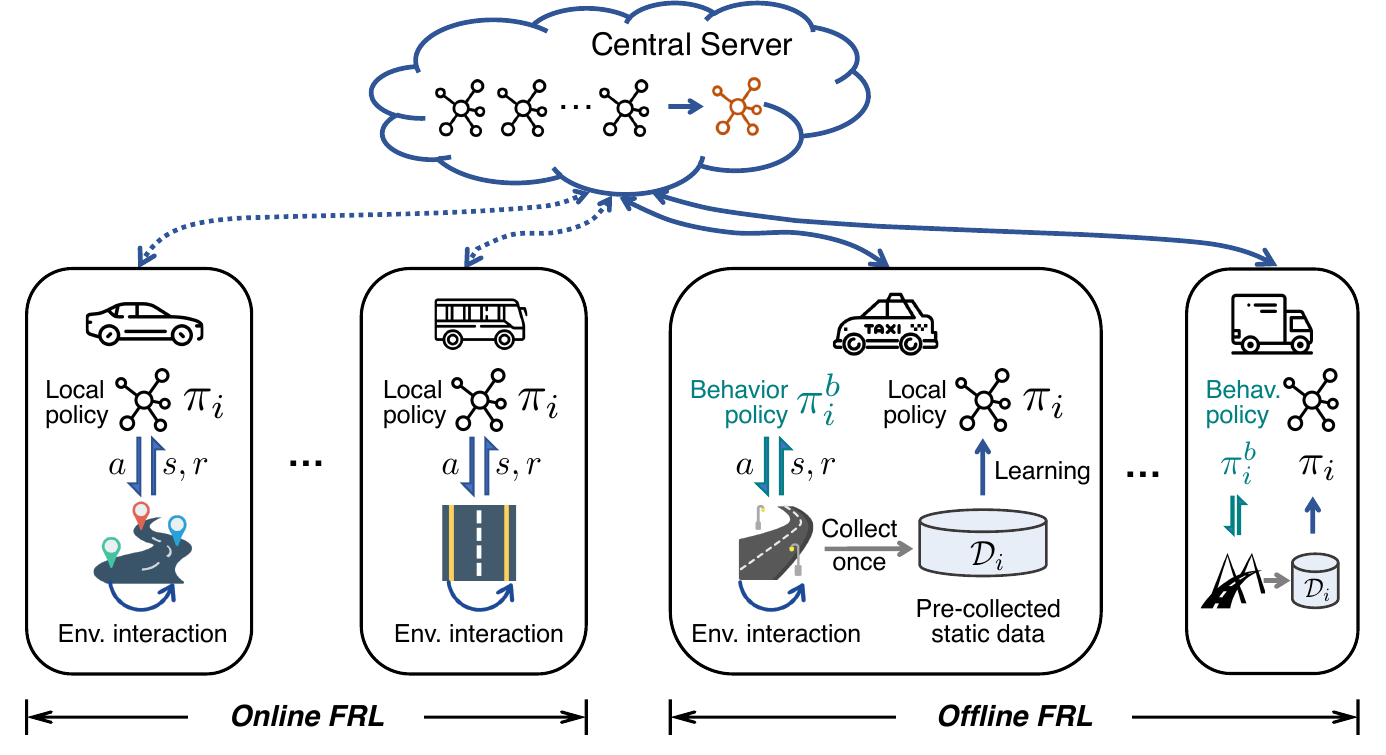}
    \vspace{-2.25em}
    \caption{Online FRL vs Offline FRL.}
    \vspace{-2.0em}
\label{fig:system_model}
\end{figure}

Unfortunately, achieving effective offline FRL is highly nontrivial due to the intrinsic issue of \emph{two-tier distributional shifts}. 
On one hand, the state-action distribution of the local dataset, on which each agent updates policies, can deviate increasingly from that of the learned policy, leading to severe extrapolation errors during offline learning~\cite{levine2020offline}.  On the other hand, significant heterogeneity may exist in the data distributions across agents, stemming from the potential variations in data collection policies, which would introduce substantial bias during local policy updating~\cite{kairouz2021advances}. If left unchecked, these two-tier distributional mismatches will hinder the effectiveness and efficiency of policy optimization in offline FRL, limiting its practical applicability and potential impact. 



In this paper, we introduce a new offline FRL algorithm,  named \emph{\underline{d}oubly \underline{r}egularized federated offline \underline{p}olicy \underline{o}ptimization} (\alg{}), that leverages dual regularization, one based on the local behavioral state-action distribution and the other on the global aggregated policy, to judiciously solve the above-mentioned dilemma. Specifically, the first regularization can incorporate \emph{conservatism} into the local learning policy to ameliorate the effect of extrapolation errors. The second can confine the local policy around the global policy to impede over-conservatism induced by the first regularizer and enhance the utilization of the aggregated information. We theoretically analyze the effects of the dual regularization and demonstrate  that by striking the right balance thereof, \alg{} can effectively combat the distributional mismatches and achieve strict improvement in policy updating. 

The main contribution of this paper is summarized below:
\begin{itemize}[topsep=1pt]
    \item We devise a doubly regularized federated offline policy optimization algorithm, named \alg{}, capable of  extracting valuable policies from distributed static data without any online environmental interaction.
    \item We rigorously characterize the impacts of the introduced dual regularization and establish strict policy improvement guarantees for \alg{} via balancing the right tradeoff between the two regularizers.
    \item Our extensive experimental results showcase that \alg{} significantly outperforms the baselines on standard offline RL benchmarks in terms of performance and communication efficiency (often within 20 communication rounds).
\end{itemize}

\section{Related Work}
\label{sec:related_work}


A growing body of literature has investigated FRL, with the aim of safely co-train policies for overcoming data sharing constraints~\cite{wang2023attrleaks}. Nadiger et al.~\cite{nadiger2019federated} propose an FRL approach that combines \texttt{DQN}~\cite{hasselt2010double} and \texttt{FedAvg}~\cite{mcmahan2017communication} to achieve personalized policies for individual players in the game of Pong. Liu et al.~\cite{liu2019lifelong} propose a continual FRL method specifically designed for robot navigation problems, allowing new robots to utilize the shared policy to speed up training. Lim et al.~\cite{lim2020federated} propose an FRL algorithm on top of Proximal Policy Optimization (\texttt{PPO})~\cite{schulman2017proximal} and \texttt{FedAvg}. Liang et al.~\cite{liang2019federated} retrofit Deep Deterministic Policy Gradient (\texttt{DDPG})~\cite{lillicrap2015continuous} for \texttt{FedAvg} in the context of autonomous driving. Cha et al.~\cite{cha2019federated} introduce a privacy-preserving variant of policy distillation, where they exchange the proxy experience memory (including pre-arranged states and time-averaged policies) instead of raw data. Zhuo et al.~\cite{zhuo2020federated} propose a two-agent FRL framework with discrete state-action spaces, where agents share local encrypted (by Gaussian differentials) Q-values and alternately update the global Q-network by multilayer perceptron (MLP). Anwar
and Raychowdhury~\cite{anwar2021multi} study the adversarial attack issue in FRL, with the goal of training a unified policy for individual tasks. Khodadadian et al.~\cite{khodadadian2022federated} establish convergence guarantees for the combination of \texttt{FedAvg} and classical \texttt{Q-learning}. Fan et al.~\cite{fan2021fault} develop a fault-tolerant FRL algorithm that can handle scenarios where a certain percentage of participating agents may experience random system failures or act as adversarial attackers. They  analyze the sample complexity of the proposed algorithm. However, these existing FRL works concentrate on the online setting where agents require extensive interaction with the environment. In contrast, the focus of this paper lies in offline FRL that collaboratively conducts policy learning without active online exploration. 

Recently, Zhou et al.~\cite{zhou2022federated} explore a similar formulation for the dynamic treatment regime (DTR) problem, but under a strong assumption of episodic linear MDP. 
Shen et al.~\cite{shen2023distributed} and, Rengarajan et al.~\cite{rengarajan2023federated} introduce independent offline federated policy optimization methods, whereas they neglect the important issue of two-tier distributional shifts along with the delicate tradeoff involved. These algorithms may fail to capture global information during local updates, leading to poor learning efficiency, particularly in high-dimensional environments. Of note, there has been a number of efforts that focus on single-agent offline RL in the RL realm~\cite{levine2020offline,kumar2020conservative,kostrikov2021offline,yu2021combo}. However, it has been reported that the naive federation approach (using the offline RL objective as the local objective and applying \texttt{FedAvg}) suffers from an unsatisfactory performance, even worse than individual offline RL~\cite{rengarajan2023federated}.

\section{Background and Challenge}
\label{sec:background}

In this section, we introduce the necessary background and elaborate on the concrete challenges to motivate our method.

\subsection{Markov Decision Processes and Reinforcement Learning}

The problem of RL is typically formulated as a Markov Decision Process (MDP), specified by $M\doteq\langle\mathcal{S},\mathcal{A},\dy,R,\mu_0,\gamma\rangle$ with state space $\mathcal{S}$,  action space $\mathcal{A}$, transition dynamics  $\dy:\mathcal{S}\times\mathcal{A}\rightarrow\mathcal{P}(\mathcal{S})$, reward function $R:\mathcal{S}\times\mathcal{A}\rightarrow\mathcal{P}([0,R_{\max}])$, initial state distribution $\mu_0$, and discount factor $\gamma\in(0,1)$. A policy maps states to distributions over actions, represented as $\pi:\mathcal{S}\rightarrow\mathcal{P}(\mathcal{A})$. The occupancy measure of policy $\pi$ is defined as $\rho^{\pi}(s,a)\doteq(1-\gamma)\sum^{\infty}_{h=0}\gamma^h\Pr(s_h=s | \dy,\pi,\mu)\pi(a|s)$. The objective of RL can be expressed as maximizing expected cumulative rewards: 
\begin{align}
    \label{eq:objective}
    \max_{\pi}J(M,\pi)\doteq\mathbb{E}\Bigg[\sum^{\infty}_{h=0}\gamma^{h} r(s_h,a_h)\mid T,\mu,\pi_\theta\Bigg],
\end{align}
or equivalently $\max_{\pi\in\Pi} \mathbb{E}_{s,a\sim\rho^{\pi}}[R(s,a)/(1-\gamma)]$~\cite{sutton2018reinforcement}. Additionally, we define the action-value function (Q-function) of policy $\pi$ as $Q^{\pi}(s,a)\doteq \mathbb{E}[\sum^{\infty}_{t=0}\gamma^t R(s_t,a_t)| s_0=s,a_0=a]$.

\subsection{Offline Federated Reinforcement Learning}

As depicted in \cref{fig:system_model}, offline FRL
solves Problem (\ref{eq:objective}) in a distributed and offline fashion, where ${n}$ distributed agents federatively build a policy under the orchestration of a central server, without sharing their raw trajectories or further interacting with the environment. Let $\mathcal{D}_i\doteq\{(s_{i,j},a_{i,j},r_{i,j},s'_{i,j})\}^{D_i}_{j=1}$ represent the private local dataset of agent $i\in\{1,\dots,n\}$, where $(s_{i,j},a_{i,j},r_{i,j},s'_{i,j})$ is a transition tuple sampled from an unknown policy. The goal is to find the best possible policy from distributed offline datasets, $\{\mathcal{D}_i\}^n_{i=1}$, to maximize Problem (\ref{eq:objective}). We define $\mathcal{D}_i(s,a)\doteq\{(s',a')\in\mathcal{D}_i\mid s'=s,a'=a\}$ and the \emph{behavioral policy} induced by $\mathcal{D}_i$ as follows:
\begin{align}
    \pi^b_i(a|s)\doteq
    \begin{cases}
        \frac{\sum_{(\tilde s,\tilde a)\in\mathcal{D}_i}\1((\tilde s,\tilde a)=(s,a))}{\sum_{\tilde s\in\mathcal{D}_i}\1(\tilde s=s)},&\text{if}~s\in\mathcal{D}_i;\\
        \frac{1}{|\mathcal{A}|},&\text{else}.
    \end{cases}
\end{align}


\subsection{Challenges}
\label{sec:challenges}

\begin{figure}[t]
    \centering
    \subfigure{\includegraphics[width=0.475\columnwidth]{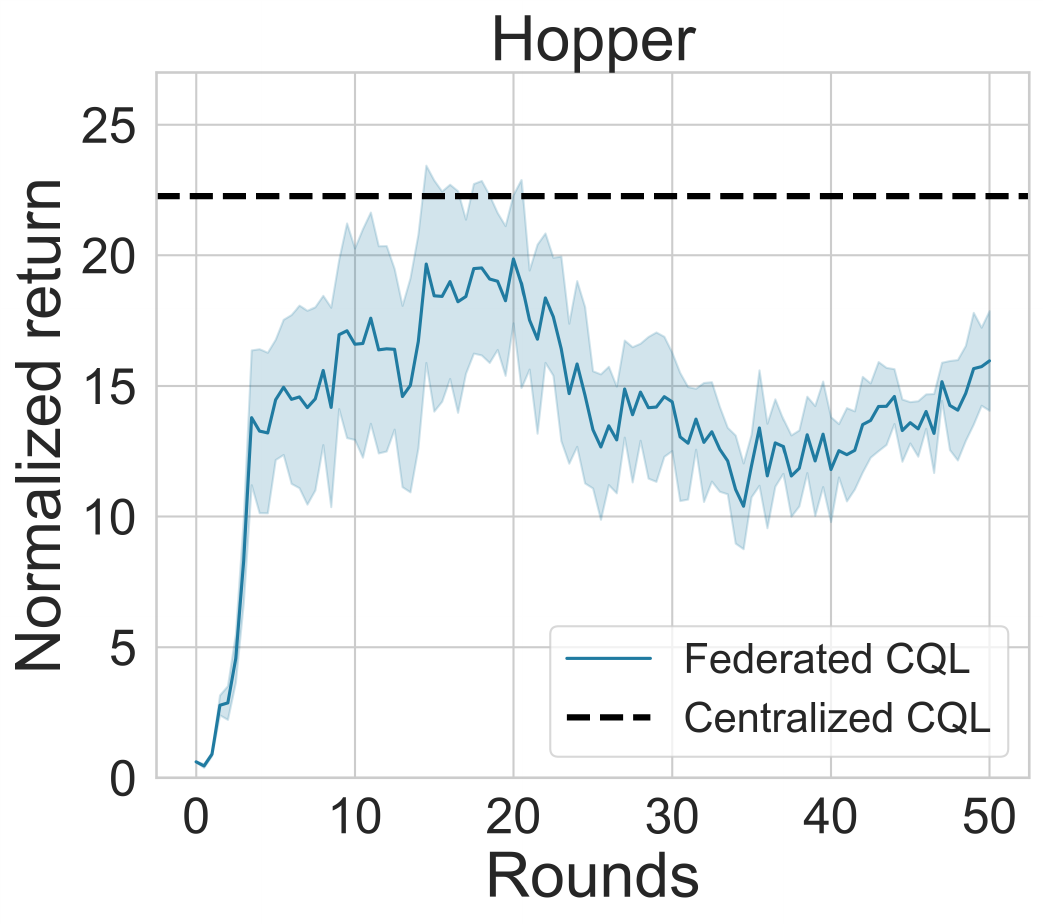}}
    \hspace{1pt}
    \subfigure{\includegraphics[width=0.475\columnwidth]{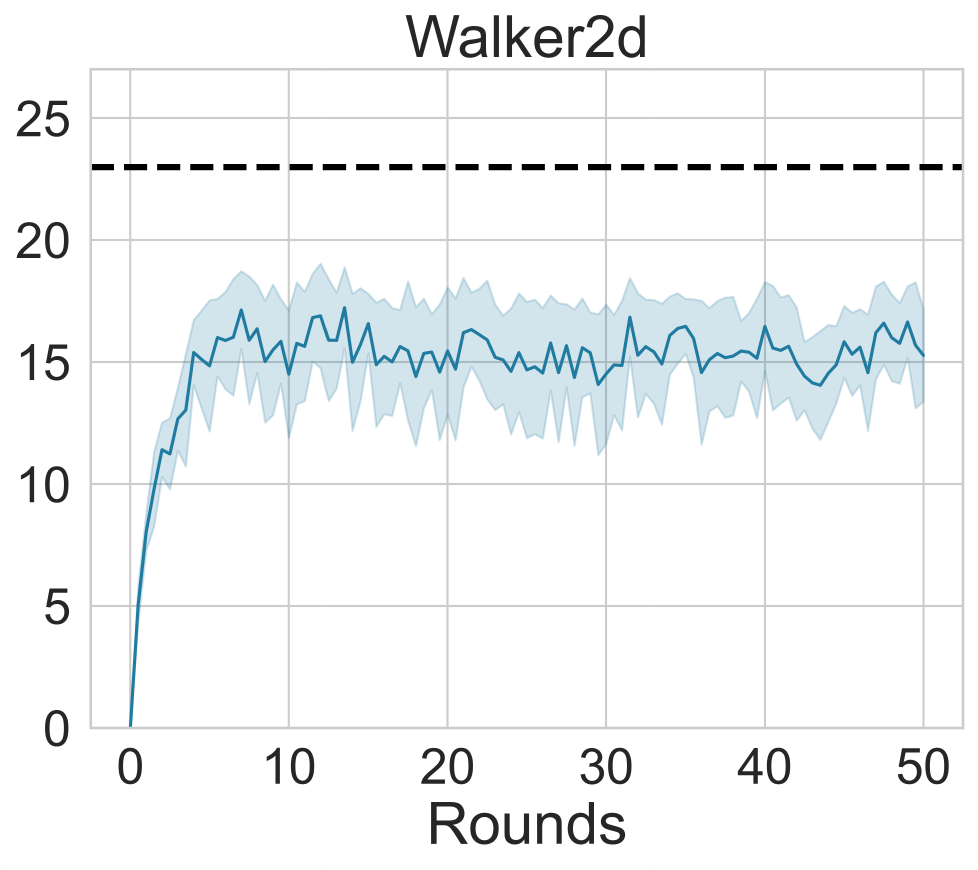}}
    \vspace{-1.0em}
    \caption{Performance comparison of centralized \texttt{CQL} and naive federated \texttt{CQL} on two challenging MuJoCo tasks. Federated \texttt{CQL} follows \texttt{FedAvg} by using the \texttt{CQL} objective as the local objective. In Federated \texttt{CQL}, the number of agents is set as 10, and each agent has 5 \texttt{medium-replay} trajectories sampled from D4RL~\cite{fu2020d4rl}. The centralized \texttt{CQL} uses all 50 trajectories.}
    \vspace{-1.0em}
    \label{fig:motivating}
\end{figure}

A fundamental challenge in offline FRL is the \emph{locally distributional shift}, that is, the deviation between state-action distribution $\rho^{\pi}(s,a)$ of the current learned policy and that of $\mathcal{D}_i$. As the locally learned policies increasingly deviate from $\pi^b_i$, this deviation would be exacerbated over the course of learning and could lead to significant extrapolation error in policy evaluation.
Intuitively, a straightforward solution for this issue is to repurpose well-established batch RL approaches (e.g., \texttt{CQL}~\cite{kumar2020conservative} and \texttt{MOPO}~\cite{yu2020mopo}) to operate in FL settings. For example, each agent can carry out multiple \texttt{CQL} updating steps, followed by periodically aggregating the parameters of policies and critics in the server. Unfortunately, as demonstrated in \cref{fig:motivating} and \cite[Section 4]{rengarajan2023federated}, the solution easily suffers from instability and poor performance, and this sort of combinations may perform even worse than individual offline RL in many tasks. The underlying reason is that the inherent \emph{conservatism} in offline RL methods would easily underestimate locally out-of-distribution state-actions and overfit local datasets, failing to leverage valuable global information from aggregation.\footnote{The conservatism means compelling the learning policy to stay close to the offline data manifold, which is a commonly used principle in the algorithmic design to deal with the extrapolation errors in offline RL.} It reveals that, in addition to the locally distributional mismatch, offline FRL should carefully deal with the \emph{global distributional shift} caused by the data heterogeneity among agents. The presence of these distributional mismatches makes federated offline policy optimization highly challenging.

\begin{figure}[ht]
    \centering
    \vspace{-.5em}
    \includegraphics[width=0.9\linewidth]{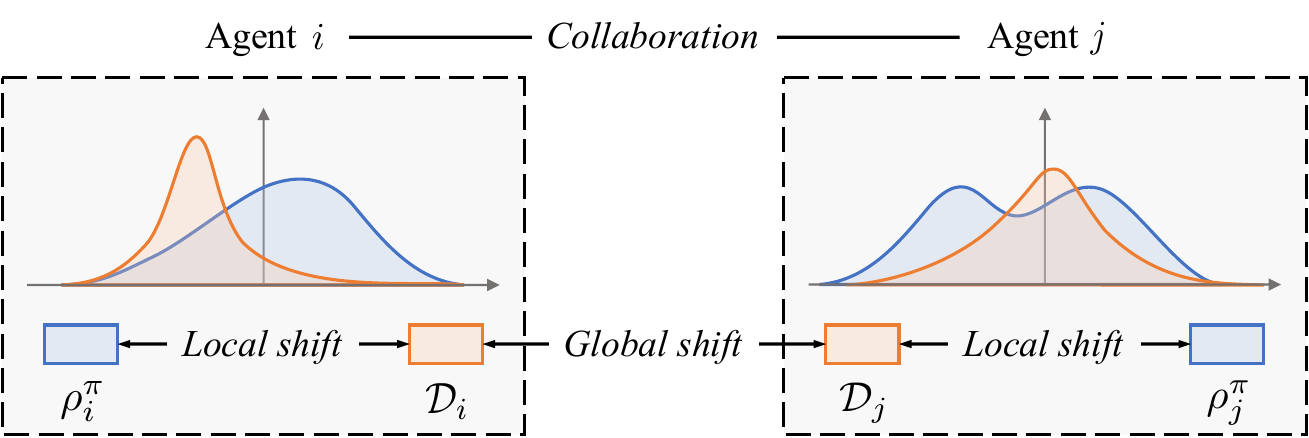}
    \vspace{-0.5em}
    \caption{An illustration of the two-tier distributional shifts. $\mathcal{D}_i$ and $\mathcal{D}_j$ represent the state-action (data) distributions in local datasets. $\rho^\pi_i$ and $\rho^\pi_j$ denote the state-action distributions of local policies $\pi_i$ and $\pi_j$ in the environment.}
    \vspace{-0.5em}
    \label{fig:shifts}
\end{figure}

\section{Federated Offline Policy Optimization}

The aforementioned dilemma implies a subtle tradeoff that has to be delicately calibrated in the policy optimization of offline FRL. To combat the local distributional shift, the agent's updating policy should be performed more \emph{on local data support}, because the local data has no knowledge of the out-of-distribution states and actions. At the same time, the local policy also requires prudently \emph{escaping the local data manifold} to incorporate the valuable knowledge from other agents. To this end, inspired by the recent progress in meta-RL~\cite{lin2022model}, we next present \emph{\underline{d}oubly \underline{r}egularized federated offline \underline{p}olicy \underline{o}ptimization} (\alg{}), that exploits dual regularization, one based on the local behavioral policy and the other on the global aggregated policy, to achieve this balance. 

Specifically, \alg{} alternates between local updating and global aggregation for a number of rounds, as outlined below. 

\textbf{\emph{1) Local updating:}}
Given aggregated policy $\bar{\pi}$, each agent $i\in[n]$ locally improves the policy by solving the following regularized policy optimization problem:
\begin{align}
    \label{eq:local_updating_pi}
    \max_{\pi}\mathbb{E}_{s\sim\mathcal{D}_i,a\sim\pi(\cdot|s)}\big[\widetilde{Q}^{\pi}(s,a)\big] - \lambda_{1} D(\pi,\pi^b_i) - \lambda_{2} D(\pi,\bar{\pi}),
\end{align}
where $\widetilde{Q}^\pi$ is the estimated Q-value for $\pi$, and $\lambda_1,\lambda_2\ge0$ are non-negative coefficients. In \cref{eq:local_updating_pi}, we introduce a weighted interpolation of two regularization terms with behavioral policy $\pi^b_i$ and global policy $\bar{\pi}$. The first regularization, $D(\pi_i,\pi^b_i)$, incorporates conservatism into the local learning policy to ameliorate the effect of extrapolation errors. The second one, $D(\pi_i,\bar{\pi})$, constrains the local policy around the global policy to impede \emph{over-conservatism} induced by the first regularizer and enhance the utilization of aggregated information. 

Regarding $\widetilde{Q}^\pi$, we obtain its value by repeating the following off-policy evaluation process, as in \cite{kumar2020conservative}:
\begin{align}
    \label{eq:cql}
    \widetilde{Q}^\pi \leftarrow& \arg \min_{Q} \frac{1}{2} \mathbb{E}_{s,a\sim \mathcal{D}_i}\Big[\big(Q(s, a)-\widetilde{\bl}^\pi_i \widetilde{Q}^\pi(s, a)\big)^{2}\Big]\nonumber\\
    +&\beta\Big(\mathbb{E}_{s,a \sim\tilde{\rho}^\pi_i}\big[Q(s, a)\big]-\mathbb{E}_{s, a \sim \mathcal{D}_i}\big[Q(s, a)\big]\Big).
\end{align}
where $\tilde{\rho}^\pi_i(s,a)\doteq \mathcal{D}_i(s)\pi(a|s)$, and $\beta\ge0$ is a coefficient. $\widetilde{\bl}^\pi_i$ represents the empirical Bellman operator, defined by
\begin{align}
    \label{eq:empirical_bellman}
    \widetilde{\bl}^\pi_i Q(s,a)\doteq \mathbb{E}_{s,a,r,s'\sim\mathcal{D}_i} \big[r +  \gamma\mathbb{E}_{a'\sim\pi(\cdot|s')}[Q(s',a')]\big],
\end{align}
In \cref{eq:cql}, the first quadratic term is the standard Bellman updating, and the second term penalizes the action-value for actions not observed in $\mathcal{D}_i$ to alleviate the value overestimation on out-of-distribution state-actions.

After local updates, each agent transmits the local policy to the server for policy aggregation.

\textbf{\emph{2) Global aggregation:}}
Once receiving local policies from agents, the server carries out aggregation:
\begin{align}
    \label{eq:aggregation}
    \bar{\pi}\leftarrow\frac{1}{n}\sum\nolimits^n_{i=1}\pi_i,
\end{align}
where we slightly abuse notation and denote $\bar{\pi},\pi_i$ as the corresponding policy parameters. Then, the server sends the aggregated policy back to agents and starts the next round.

The pseudocode of \alg{} is outlined in Alg.~\ref{alg:drpo}. Thanks to the regularizers, the processes of policy evaluation and policy improvement are expected to be carried out in multiple steps, while retaining the global information throughout (thousands of local gradient steps per round work well in our experiment). As a consequence, \alg{} enjoys notable communication efficiency (often converges within 20 rounds in experiments). 

\begin{algorithm}[t]
    \LinesNumbered
    \KwIn{$\lambda_1$, $\lambda_2$, $\beta$, $\{\mathcal{D}\}^n_{i=1}$}
    Serve initializes policy $\bar{\pi}$ and sends it to all agents\;
    \For{$t=1$ \KwTo $T$}{
        \For{$i=1$ \KwTo $N$}{
            \tcp{Local updating}
            Agent $i$ locally updates the policy and critic multiple steps according to \cref{eq:cql,eq:local_updating_pi}\;
            Agent $i$ transmits updated $\pi_i$ to the server\;
        }
        \tcp{Global aggregation}
        After receiving local policies, server aggregates them by \cref{eq:aggregation} and distributes to all agents\;
    }
    \label{alg:drpo}
    \caption{\alg{}}
\end{algorithm}





\section{Theoretical Analysis}
\label{sec:theory}

This section seeks to answer the following key question in theory: 
can \alg{} counteract the distributional shifts and achieve effective policy improvement in each round? 



\subsection{Notations and Assumptions}
\label{sec:assumptions}

We begin by introducing the key notations and assumptions used in the analysis. Notably, since the following analysis operates in each agent, we omit subscript $i$ for conciseness. For any $s,s'\in\mathcal{S}$ and $a\in\mathcal{A}$, define the empirical transition matrix and reward function as follows:
\begin{align}
    \widetilde{\dy}(s'|s,a)&\doteq \sum_{(s_i,a_i,s'_i)\in\mathcal{D}}\frac{\1((s_i,a_i,s'_i)=(s,a,s'))}{|\mathcal{D}(s,a)|},
    \\
    \widetilde{R}(s,a)&\doteq\sum_{(s_i,a_i,r_i)\in\mathcal{D}}\frac{r_i\cdot\1((s_i,a_i)=(s,a))}{|\mathcal{D}(s,a)|},
\end{align}
where we assume the cardinality of a state-action pair in $\mathcal{D}$ is non-zero, by setting $|\mathcal{D}(s,a)|/|\mathcal{D}|\ge\delta$ when $(s,a)\notin\mathcal{D}$, to prevent any trivial bound~\cite{kumar2020conservative}. We impose the standard concentration assumption in analyzing  RL algorithms \cite{agarwal2019reinforcement}. 
\begin{assumption}
    \label{asmp:concentration}
    For any $s,a\in\mathcal{D}$, with probability greater than $1-\delta$, the following relationship holds:
    \begin{align}
        \big\|\widetilde{\dy}(\cdot|s,a)-\dy(\cdot|s,a)\big\|_{1} &\leq \frac{C_{\dy, \delta}}{\sqrt{|\mathcal{D}(s,a)|}},
        \\
        \big\|\widetilde{R}(s,a)-R(s,a)\big\|_1 &\leq \frac{C_{R, \delta}}{\sqrt{|\mathcal{D}(s, a)|}},
    \end{align}
    where $C_{\dy, \delta}$ and $C_{R, \delta}$ are positive constants depending on $\delta$ via a $\sqrt{\log(1/\delta)}$ dependency.
\end{assumption}
\cref{asmp:concentration} reasons about the estimation errors of empirical transitions and rewards in terms of data sizes. Based on this and the fact $Q(s,a)\le R_{\max}/(1-\gamma)$, we can obtain the following relationship with probability greater than $1-\delta$:
\begin{align}
    &\big|\widetilde{\bl}^\pi Q(s,a) - \bl^\pi Q(s,a)\big|\nonumber\\
    \le\;& \Big|\widetilde{R}(s,a) - R(s,a) + \gamma\sum\nolimits_{s'}\big(\widetilde{\dy}(s'|s,a)-P(s'|s,a)\big)\nonumber\\
    &\cdot\mathbb{E}_{a'\sim\pi(\cdot|s')}\big[Q(s',a')\big]\Big|\tag{from definitions of $\widetilde{\bl}^\pi,\bl^\pi$}
    \\
    \le\;&\frac{(1-\gamma)C_{R, \delta} + \gamma R_{\max}C_{\dy, \delta}}{(1-\gamma)\sqrt{|\mathcal{D}(s,a)|}},
\end{align}
which bounds the estimation error induced by the empirical Bellman operator. Throughout this section, we let $D$ represent the maximum total variation distance, i.e.,
\begin{align}
    D(\pi_1,\pi_2)\doteq\max_s \TV(\pi_1(\cdot\vert s)\Vert \pi_2(\cdot\vert s)).
\end{align}
In addition, when clear from the context, we use $\pi$ to denote the optimal solution to Problem (\ref{eq:local_updating_pi}).

\subsection{Main Results: Strict Policy Improvement}
\label{sec:main_result}

Based on \cref{asmp:concentration}, we have the following policy improvement result for \alg{}.

\begin{theorem}
    \label{thm:policy_improvement}
    Suppose that \cref{asmp:concentration}, $(1-\delta_\pi)\lambda_2<\lambda_1<\lambda_2$, and $\delta_\pi>0$ (defined in Eq.~(\ref{eq:delta_pi})) holds. Then, the optimal solution $\pi^*$ of Problem (\ref{eq:local_updating_pi}) satisfies
    \begin{align}
        J(M,\pi^*)>\max\{J(M,\bar{\pi}),J(M,\pi^b)\},
    \end{align}
    with probability great than $1-2\delta$ and a sufficient large $\lambda_2$.
\end{theorem}

\noindent\textbf{\emph{Remarks.}} \cref{thm:policy_improvement} reveals that the local policy learned by \alg{} exhibits `strict' performance improvement over both local behavioral policy and global policy during each local updating phase. The result indicates that \alg{} can extract valuable policy from distributed offline data while accommodating data of different characteristics. Note that effective policy improvement is contingent upon a relatively large $\lambda_2$. This finding aligns with our empirical observations in \cref{fig:motivating} and underscores the crucial role of a constraint associated with the aggregated model.

In subsequent subsections, we provide proof for \cref{thm:policy_improvement}.

\subsection{Improvement over aggregated policy}

First, we show that \alg{} can achieve strict improvement over the aggregated policy, $\bar{\pi}$. Denote $d(s,a)\doteq|\mathcal{D}(s,a)|/|\mathcal{D}|$ as the empirical state-action distribution of local dataset $\mathcal{D}$. Then, from \cref{eq:cql} , the Q-function found by approximate dynamic programming in iteration $k$ can be obtained by setting the derivation of \cref{eq:cql} to 0:
\begin{align}
    Q^{k+1}(s,a) = \widetilde{\bl}^\pi Q^k(s,a) - \beta\cdot\frac{\tilde{\rho}^\pi(s,a)-d(s,a)}{d(s,a)},
\end{align}
which penalizes the Bellman equation by an additional penalty to alleviate the overestimation on out-of-distribution actions. We define the expected penalty on the Q-value under $\tilde{\rho}^\pi(s,a)$ as
\begin{align}
    \label{eq:eq3}
    {g}(\tilde{\rho}^\pi)\doteq\mathbb{E}_{(s,a)\sim\tilde{\rho}^\pi}\bigg[\frac{\tilde{\rho}^\pi(s,a)-d(s,a)}{d(s,a)}\bigg].
\end{align}
Then, it can be seen that Problem (\ref{eq:local_updating_pi}) is equivalent to the following problem:
\begin{align}
    \max_\pi J(\widetilde{M},\pi) - \frac{\beta{g}(\tilde{\rho}^\pi)}{1-\gamma} - \lambda_1 D(\pi,\pi^b) - \lambda_2 D(\pi,\bar{\pi}),
\end{align}
where  $\widetilde{M}\doteq\langle\mathcal{S},\mathcal{A},\widetilde{\dy},\widetilde{R},\mu,\gamma\rangle$ is the empirical MDP induced by $\mathcal{D}_i$. Since the learning is fully carried out on fixed data $\mathcal{D}$, we require the relationship between the returns of policy $\pi$ on empirical MDP $\widetilde{M}$ and underlying MDP $M$. Accordingly, we provide the following lemma adapted from \cite[Lemma A.2]{yu2021combo}.
\begin{lemma}
    \label{lem:interval}
    For any policy $\pi$, the following fact holds with the probability greater than $1-\delta$:
    \begin{align}
        J(M,\pi) - \eta \le J(\widetilde{M},\pi) \le J(M,\pi) + \eta,
    \end{align}
    where $\eta_i$ is denoted as
    \begin{align*}
        \eta\doteq\;&\frac{\gamma }{1-\gamma}\Big|\mathbb{E}_{s,a\sim\rho^\pi,a'\sim\pi(\cdot|s')}\Big[ \sum_{s'}\big(\widetilde{\dy}(s'|s,a)-P(s'|s,a)\big)\nonumber\\
        &\cdot Q^\pi(s',a')\Big]\Big| + \frac{1}{1-\gamma}\mathbb{E}_{s,a\sim\rho^\pi}\Big[\big| \widetilde{R}(s,a) - R(s,a) \big|\Big].
    \end{align*}
\end{lemma}
\begin{proof}
    Letting $M_1=M_2=\widetilde{M}$ and $f=1$ in \cite[Lemma A.2]{yu2021combo} yield the result.
\end{proof}
\cref{lem:interval} bounds the difference between the returns on empirical MDP and true MDP from above and below. Built on \cref{lem:interval} and \cite[Proof of Theorem 2]{yu2021combo}, the following holds with probability greater than $1-\delta$:
\begin{align}
    \label{eq:eq1}
    J(\widetilde{M},\pi)\le J(M,\pi) +  \tilde{\eta},
\end{align}
where $\tilde{\eta}$ is defined as
\begin{align}
    \tilde{\eta} \doteq\;&\frac{2\gamma R_{\max} C_{P,\delta}}{(1-\gamma)^2}\mathbb{E}_{s\sim d}\bigg[ \sqrt{D_\mathrm{CQL}(s;\pi,\pi^b)|\mathcal{A}|/|\mathcal{D}|} \bigg]\nonumber\\
    & + \frac{C_{R,\delta}}{1-\gamma}\mathbb{E}_{s,a\sim\rho^{\pi}}\Big[ 1/\sqrt{|\mathcal{D}(s,a)|}\Big],
\end{align}
with $D_\mathrm{CQL}(s;\pi_1,\bar{\pi})\doteq 1 + \sum_a \pi_1(a|s)(\pi_1(a|s)/\bar{\pi}(a|s)-1)$. Similarly, we can bound the return of $\bar{\pi}$ in the empirical MDP from below with respect to its return in the underlying MDP with probability greater than $1-\delta$:
\begin{align}
    \label{eq:eq2}
    J(\widetilde{M},\bar{\pi})\ge J(M,\bar{\pi}) - \bar{\eta},
\end{align}
where $\bar{\eta}$ is defined as
\begin{align}
    \bar{\eta} \doteq\;&\frac{2\gamma R_{\max} C_{P,\delta}}{(1-\gamma)^2}\mathbb{E}_{s\sim d}\bigg[ \sqrt{D_\mathrm{CQL}(s;\bar{\pi},\pi^b)|\mathcal{A}|/|\mathcal{D}(s)|} \bigg]\nonumber\\
    & + \frac{C_{R,\delta}}{1-\gamma}\mathbb{E}_{s,a\sim\rho^{\bar{\pi}}}\Big[ 1/\sqrt{|\mathcal{D}(s,a)|}\Big].
\end{align}
Based on \cref{eq:eq1,eq:eq2}, if $\pi$ is the optimal solution to Problem (\ref{eq:local_updating_pi}), the following holds:
\begin{align}
    &J(M,\pi) +  \tilde{\eta} - \frac{\beta{g}(\tilde{\rho}^{\pi})}{1-\gamma} - \lambda_1 D(\pi,\pi^b) - \lambda_2 D(\pi,\bar{\pi})\nonumber\\
    \ge& J(\widetilde{M},\pi) - \frac{\beta{g}(\tilde{\rho}^{\pi})}{1-\gamma} - \lambda_1 D(\pi,\pi^b) - \lambda_2 D(\pi,\bar{\pi})\nonumber\\
    \ge& J(\widetilde{M},\bar{\pi}) - \frac{\beta{g}(\tilde{\rho}^{\bar{\pi}})}{1-\gamma} - \lambda_1 D(\bar{\pi},\pi^b)\tag{from optimality of $\pi$}\\
    \ge&J(M,\bar{\pi}) - \bar{\eta} - \frac{\beta{g}(\tilde{\rho}^{\bar{\pi}})}{1-\gamma} - \lambda_1 D(\bar{\pi},\pi^b).
    \label{eq:eq9}
\end{align}
This gives us a lower bound on $J(M,\pi)$ in terms of $J(M,\bar{\pi})$:
\begin{align}
    J(M,\pi) &\ge J(M,\bar{\pi}) - \bar{\eta} - \tilde{\eta}  + \bar{\sigma},
    \label{eq:eq8}
\end{align}
where $\bar{\sigma}$ is defined as
\begin{align}
    \bar{\sigma}&\doteq \frac{\beta\big({g}(\tilde{\rho}^{\pi}) - {g}(\tilde{\rho}^{\bar{\pi}})\big)}{1-\gamma}+ \lambda_1 D(\pi,\pi^b) \nonumber\\
    &+ \lambda_2 D(\pi,\bar{\pi}) - \lambda_1 D(\bar{\pi},\pi^b).
    \label{eq:sigma_bar}
\end{align}
To ensure policy improvement over global policy $\bar{\pi}$, we proceed to establish $\bar{\sigma}- \bar{\eta} - \tilde{\eta}>0$. Regarding ${g}(\tilde{\rho}^{\pi}) - {g}(\tilde{\rho}^{\bar{\pi}})$ in \cref{eq:sigma_bar}, using the definition of ${g}(\cdot)$ yields 
\begin{align}
    &\big|{g}(\tilde{\rho}^{\pi}) - {g}(\tilde{\rho}^{\bar{\pi}})\big|\nonumber\\
    =&\Bigg|\sum_{s,a}\tilde{\rho}^{\pi}(s,a)\bigg(\frac{\tilde{\rho}^{\pi}(s,a)}{d(s,a)}-1\bigg)-\tilde{\rho}^{\bar{\pi}}(s,a)\bigg(\frac{\tilde{\rho}^{\bar{\pi}}(s,a)}{d(s,a)}-1\bigg)\Bigg|\nonumber\\
    \le&\Bigg|\sum_{s,a}\tilde{\rho}^{\pi}(s,a)\bigg(\frac{\tilde{\rho}^{\pi}(s,a)}{d(s,a)}-1\bigg)-\tilde{\rho}^{\bar{\pi}}(s,a)\bigg(\frac{\tilde{\rho}^{\pi}(s,a)}{d(s,a)}-1\bigg)\Bigg|\nonumber\\
    +&\Bigg|\sum_{s,a}\tilde{\rho}^{\bar{\pi}}(s,a)\bigg(\frac{\tilde{\rho}^{\pi}(s,a)}{d(s,a)}-1\bigg)-\tilde{\rho}^{\bar{\pi}}(s,a)\bigg(\frac{\tilde{\rho}^{\bar{\pi}}(s,a)}{d(s,a)}-1\bigg)\Bigg|\tag{from the triangle inequality}\\
    =&\Bigg|\sum_{s,a}\Big(\tilde{\rho}^{\pi}(s,a)-\tilde{\rho}^{\bar{\pi}}(s,a)\Big)\cdot\frac{\tilde{\rho}^{\pi}(s,a)-d(s,a)}{d(s,a)}\Bigg|\nonumber\\
    +&\Bigg|\sum_{s,a}\tilde{\rho}^{\bar{\pi}}(s,a)\cdot\bigg(\frac{\tilde{\rho}^{\pi}(s,a)}{d(s,a)}-\frac{\tilde{\rho}^{\bar{\pi}}(s,a)}{d(s,a)}\bigg)\Bigg|\tag{arranging terms}\\
    \le&\frac{2}{\delta}\sum_{s,a}\bigg|\tilde{\rho}^{\pi}(s,a)-\tilde{\rho}^{\bar{\pi}}(s,a)\bigg|,
\end{align}
where the last inequality holds from $d(s,a)\ge\delta$ (see \cref{sec:assumptions}). Letting $L\doteq4/\delta$, we have
\begin{align}
    \big|{g}(\tilde{\rho}^{\pi}) - {g}(\tilde{\rho}^{\bar{\pi}})\big|\le L\TV(\tilde{\rho}^{\pi}\|\tilde{\rho}^{\bar{\pi}}).
    \label{eq:eq10}
\end{align}
\cref{eq:eq10} gives an upper bound on the difference between the expected penalties induced by $\pi$ and $\bar{\pi}$, with respect to the difference between the corresponding state-action distributions. Drawing upon  \cref{eq:eq10,lem:occupancy_bound}, we get
\begin{align}
    \big|{g}(\tilde{\rho}^{\pi}) - {g}(\tilde{\rho}^{\bar{\pi}})\big|\le \frac{L}{1-\gamma}\max_s \TV(\pi(\cdot\vert s)\Vert \bar{\pi}(\cdot\vert s)).
    \label{eq:eq11}
\end{align}
Substituting \cref{eq:eq11} in \cref{eq:sigma_bar}, we have
\begin{align}
    \bar{\sigma}\doteq\,&  \lambda_1 D(\pi,\pi^b) + \lambda_1 D(\pi,\bar{\pi})- \lambda_1 D(\bar{\pi},\pi^b) \nonumber\\
    +\,& (\lambda_2 - \lambda_1)D(\pi,\bar{\pi}) + \frac{\beta\big({g}(\tilde{\rho}^{\pi}) - {g}(\tilde{\rho}^{\bar{\pi}})\big)}{1-\gamma}  \nonumber\\
    \ge\,& (\lambda_2 - \lambda_1)D(\pi,\bar{\pi}) +  \frac{\beta\big({g}(\tilde{\rho}^{\pi}) - {g}(\tilde{\rho}^{\bar{\pi}})\big)}{1-\gamma}  \nonumber\\
    \ge\,& \bigg(\lambda_2 - \lambda_1- \frac{\beta L}{(1-\gamma)^2}\bigg)D(\pi,\bar{\pi}).
    \label{eq:eq7}
\end{align}
Letting $B\doteq\beta L/(1-\gamma)^2$ and plugging \cref{eq:eq7} in \cref{eq:eq8}, with probability greater than $1-\delta$, we yield
\begin{align*}
    J(M,\pi) &\ge J(M,\bar{\pi}) - \bar{\eta} - \tilde{\eta}  +  (\lambda_2 - \lambda_1- B )D(\pi,\bar{\pi}).
\end{align*}
Clearly, $\bar{\eta},\tilde{\eta}$ are independent of $\lambda_1,\lambda_2$. Hence, when $\lambda_1<\lambda_2$, proper values of $\lambda_1,\lambda_2$  can lead to the performance improvement over the global policy:
\begin{align}
    J(M,\pi)>J(M,\bar{\pi}).
    \label{eq:improvement_g}
\end{align}

\subsection{Improvement over behavioral policy}

Next, we show the learned local policy achieves improvement over behavior policy $\pi^b$. Analogous to \cref{eq:eq1}, the following holds with probability greater than $1-\delta$:
\begin{align}
    J(\widetilde{M},\pi^b)\ge J(M,\pi^b) -  \eta^b,
\end{align}
where $\eta^b$ is defined as
\begin{align}
    \eta^b &\doteq\frac{2\gamma R_{\max} C_{P,\delta}}{(1-\gamma)^2}\cdot\mathbb{E}_{s\sim d}\Big[\sqrt{|\mathcal{A}|/|\mathcal{D}(s)|} \Big] \nonumber\\
    &+ \frac{C_{R,\delta}}{1-\gamma}\cdot\mathbb{E}_{s,a\sim\rho^b}\left[ 1/\sqrt{|\mathcal{D}(s,a)|}\right].
\end{align}
As in \cref{eq:eq9}, based on \cref{eq:eq1,eq:eq2}, if $\pi$ is the optimal solution to Problem (\ref{eq:local_updating_pi}), we can write
\begin{align}
    &J(M,\pi) +  \tilde{\eta} - \frac{\beta{g}(\tilde{\rho}^{\pi})}{1-\gamma} - \lambda_1 D(\pi,\pi^b) - \lambda_2 D(\pi,\bar{\pi})\nonumber\\
    \ge& J(\widetilde{M},\pi) - \frac{\beta{g}(\tilde{\rho}^{\pi})}{1-\gamma} - \lambda_1 D(\pi,\pi^b) - \lambda_2 D(\pi,\bar{\pi})\nonumber\\
    \ge& J(\widetilde{M},\pi^b) - \frac{\beta{g}(\tilde{\rho}^{b})}{1-\gamma} - \lambda_2 D(\pi^b,\bar{\pi})\nonumber\\
    \ge&J(M,\pi^b) -  \eta^b - \frac{\beta{g}(\tilde{\rho}^{b})}{1-\gamma} - \lambda_2 D(\pi^b,\bar{\pi}).
\end{align}
Then, the following fact holds:
\begin{align}
    J(M,\pi) &\ge J(M,\pi^b) - \tilde{\eta}  - \eta^b  + \sigma^b,
\end{align}
where $\sigma^b$ is defined by
\begin{align}
    \sigma^b&\doteq \frac{\beta\big({g}(\tilde{\rho}^{\pi}) - {g}(\tilde{\rho}^{b})\big)}{1-\gamma}+ \lambda_2 D(\pi,\bar{\pi}) \nonumber\\
    & + \lambda_1 D(\pi,\pi^b) - \lambda_2 D(\bar{\pi},\pi^b).
\end{align}
Recall ${g}(\rho)=\mathbb{E}_{(s,a)\sim\rho}\Big[\frac{\rho(s,a)-d(s,a)}{d(s,a)}\Big]$. As noted in \cite{kumar2020conservative,lin2022model}, ${g}(\tilde{\rho}^{b})$ is expected to be smaller than ${g}(\tilde{\rho}^{\pi})$ in practice, due to the fact $\tilde{\rho}^b(s,a)\approx d(s,a)$. Accordingly, we denote
\begin{align}
    \delta_\pi\doteq  \frac{\beta\big({g}(\tilde{\rho}^{\pi}) - {g}(\tilde{\rho}^{b})\big)}{\lambda_2(1-\gamma)D(\pi,\pi^b)},
    \label{eq:delta_pi}
\end{align}
which is positive in this case. For $\lambda_1 \ge (1-\delta_\pi)\lambda_2$, we write
\begin{align}
    \sigma^b &= \lambda_2 \delta_{\pi} D(\pi,\pi^b)+ \lambda_2 D(\pi,\bar{\pi})+ \lambda_1 D(\pi,\pi^b) - \lambda_2 D(\bar{\pi},\pi^b)\nonumber\\
    &= (\lambda_2 \delta_{\pi}+ \lambda_1) D(\pi,\pi^b)+ \lambda_2 D(\pi,\bar{\pi})  - \lambda_2 D(\bar{\pi},\pi^b)\nonumber\\
    &= \lambda_2\big(D(\pi,\pi^b) + D(\pi,\bar{\pi}) - D(\bar{\pi},\pi^b)\big)\nonumber\\
    &+\big(\lambda_1 - (1-\delta_\pi)\lambda_2\big)D(\pi,\pi^b)\nonumber\\
    &> 0.
\end{align}
Thus, when $\lambda_1 > (1-\delta_\pi)\lambda_2$, appropriate $\lambda_1,\lambda_2$ can lead to the performance improvement over the local behavior policy:
\begin{align}
    J(M,\pi)>J(M,\pi^b).
    \label{eq:improvement_b}
\end{align}
Based on \cref{eq:improvement_g,eq:improvement_b}, \cref{thm:policy_improvement} can be easily obtained.



\section{Experiments}
\label{sec:experiments}

In this section, we use experimental studies to evaluate the proposed algorithm by answering the following questions: 
\begin{enumerate}
    \item How does \alg{} perform on the standard offline RL benchmarks in comparison to the baselines?
    \item How fast does \alg{} converge?
    \item How is the scalability of \alg{}?
    \item How does \alg{} perform given different numbers of dataset sizes and local updating steps?
    \item What are the impacts of $\lambda_1$ and $\lambda_2$ on performance?
\end{enumerate}

\subsection{Experimental Setup}
\label{sec:experimental_setup}

We detail the experimental setups to ensure reproducibility.
\subsubsection{Datasets}
We evaluate our algorithms using data from a standard offline RL benchmark, D4RL \cite{fu2020d4rl} (provided at \url{https://github.com/rail-berkeley/d4rl}, under the Apache License 2.0, based on the MuJoCo simulator \cite{todorov2012mujoco}). We conduct experiments on four challenging continuous control tasks, including HalfCheetah, Walker2d, Hopper, and Ant, with three dataset types (\texttt{expert}, \texttt{medium}, and \texttt{replay-medium}). In all experiments, the trajectories are sampled randomly without replacement from the D4RL datasets.


\begin{figure}[h]
    \vspace{-0.5em}
    \centering
    \includegraphics[width=0.95\linewidth]{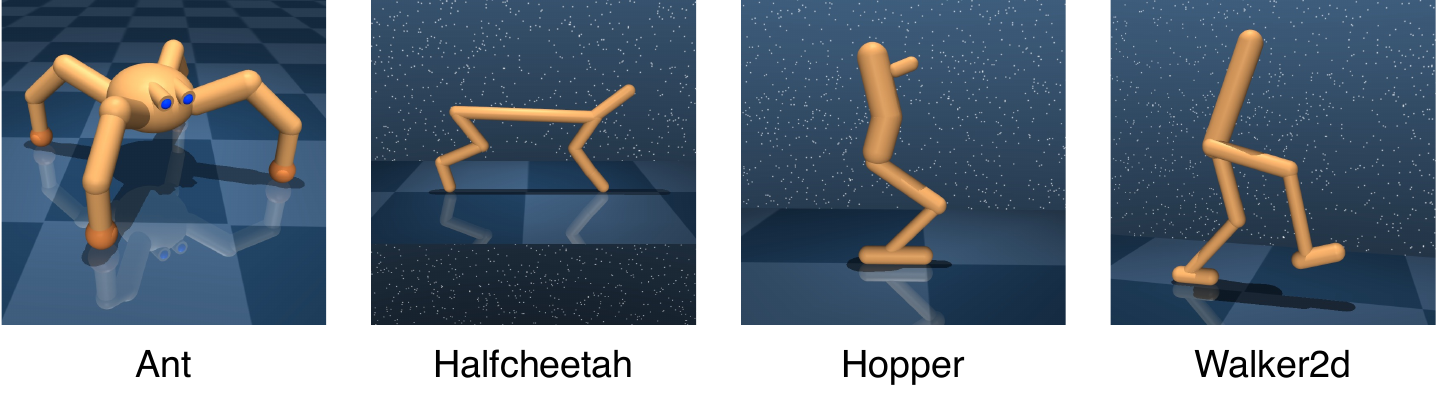}
    \vspace{-1.0em}
    \caption{Benchmark environments.}
    \vspace{-1em}
    \label{fig:mujoco}
\end{figure}

\subsubsection{Baselines}
We evaluate \alg{} against three baselines: 
\text{\emph{1) Conservative Q-Learning}} (\texttt{CQL}) \cite{kumar2020conservative}, a well-established model-free offline RL algorithm which we run individually on agents; 
\text{\emph{2) Federated Conservative Q-Learning}} (\texttt{Fed-CQL}), a combination between \texttt{FedAvg} and \texttt{CQL}, which runs CQL locally with some gradient steps followed by aggregating the policy parameters in the server; 
\text{\emph{3) Federated Behaviour Cloning}} (\texttt{Fed-BC}), the combination between \texttt{FedAvg} and Behavior Cloning, akin to \texttt{Fed-CQL}.

\begin{table}[H]
    \renewcommand\arraystretch{1.25}
    \centering
    \vspace{-1.0em}
    \caption{Hyperparameters in the experiment.} 
    \vspace{-0.5em}
    \begin{tabular}{lc} 
        \toprule
        \makebox[4.0cm][l]{Hyperparameter}                             & \makebox[1.5cm][c]{Value} \\ \midrule
        Actor learning rate                   & 3e-5  \\
        Critic learning rate                  & 3e-4  \\
        Optimizer                             & Adam  \\
        Discount factor ($\gamma$)            & 0.99  \\
        Regularization weight ($\lambda_1$) & 0.1  \\
        Regularization weight ($\lambda_2$)        & 0.2   \\
        Evaluation weight ($\beta$)        & 10.0   \\
        Batchsize                      & 256   \\ \bottomrule
    \end{tabular}
    \label{table:parameters}
\end{table}

\begin{table*}[htpb]
    \centering
    \renewcommand\arraystretch{1.25}
    \caption{Average scores on continuous control environments.}
    \begin{tabular}{@{}llrrrr@{}} 
        \toprule
        ~~\textbf{Environment}       & \textbf{Quality}  &\multicolumn{1}{c}{\textbf{\texttt{Fed-BC}}}  & \multicolumn{1}{c}{\textbf{\texttt{Fed-CQL}}} & \multicolumn{1}{c}{\textbf{\texttt{CQL}}} & \multicolumn{1}{c}{\textbf{\alg{} (ours)}} \\ \midrule
        ~~Ant            &expert        & $1329.34\pm203.29$ & $2555.60\pm200.63$ & $1923.07\pm264.01$ & $\bm{3312.55\pm371.67}$~~    \\
        ~~HalfCheetah    & expert        &$4549.21\pm991.48$ & $5413.32\pm479.85$ & $4026.87\pm477.23$ & $\bm{6928.75\pm469.80}$~~    \\
        ~~Hopper        &expert        &  $1086.63\pm244.13$ & $2167.49\pm168.09$ & $1544.23\pm186.98$ & $\bm{2767.99\pm363.72}$~~    \\
        ~~Walker2d      & expert        & $1991.56\pm307.49$ & $2986.63\pm320.35$ & $2095.37\pm290.27$ & $\bm{4201.57\pm303.76}$~~    \\
        ~~Ant            & medium        & $1146.27\pm207.75$ & $2267.02\pm543.30$ & $1707.91\pm243.87$  & $\bm{3021.74\pm452.37}$~~    \\
        ~~HalfCheetah   & medium        & $2584.25\pm484.18$ & $4137.23\pm254.72$ & $2348.82\pm339.80$& $\bm{5261.03\pm784.58}$~~    \\
        ~~Hopper        & medium        & $694.57\pm159.09$  & $1445.32\pm257.74$ & $1041.56\pm113.60$& $\bm{1854.73\pm159.89}$~~    \\
        ~~Walker2d        & medium        & $221.05\pm56.63$  & $2551.38\pm344.05$ & $1734.59\pm200.39$& $\bm{3216.83\pm227.83}$~~    \\
        ~~Ant             & medium-replay & $759.96\pm158.91$  & $1473.46\pm135.91$ & $758.74\pm89.09$ & $\bm{2156.90\pm214.06}$~~    \\
        ~~HalfCheetah    & medium-replay & $2179.53\pm499.52$ & $3573.09\pm266.89$ & $1139.83\pm142.93$ & $\bm{4396.15\pm453.39}$~~    \\
        ~~Hopper         & medium-replay & $221.26\pm66.63$  & $547.64\pm77.48$  & $364.41\pm47.31$  & $\bm{868.16\pm119.20}$~~    \\
        ~~Walker2d        & medium-replay & $269.58\pm11.60$  & $750.47\pm88.39$  & $706.61\pm80.51$  & $\bm{999.75\pm86.41}$~~    \\ \bottomrule
    \end{tabular}
    \label{table:performance}
\end{table*}

\subsubsection{Implementation}

Building on Pinsker's Inequality, we use the KL divergence instead of the total variation distance between policies in practice. The two regularizers in \cref{eq:local_updating_pi,eq:cql} can be expressed as $D(\pi,\pi^b_i) = \mathbb{E}_{s,a\in\mathcal{D}_i}[-\log\pi(a|s)]$ and $D(\pi,\bar{\pi}) = \mathbb{E}_{s\in\mathcal{D}_i,a\sim\bar{\pi}(\cdot|s)}[-\log\pi(a|s)]$, respectively. We represent the policy and critic both as a 2-layer feedforward neural network with 256 hidden units, and ReLU activation functions, where the policy network uses Tanh Gaussian outputs. The main hyperparameters used in experiments can be found in \cref{table:parameters}. We implement the code using Pytorch 1.8.1 and run experiments on Ubuntu 18.04.2 LTS with 8 NVIDIA GeForce RTX A6000 GPUs. The scores in the figures are normalized, and all results are averaged over 3 random seeds. 



\subsection{Experimental Result}
\label{sec:results}

\subsubsection{Comparative results} 
\label{sec:results_mujoco}
To answer the first and second questions, we evaluate \alg{} on various tasks with varying data qualities. We configure the number of agents and local updating steps to be 10 and 1k respectively. Each agent's dataset comprises 5 trajectories, each containing 1k transition tuples. We train all algorithms (with the same learning rate) until convergence and record the average score in \cref{table:performance}. Evidently, \alg{} yields the best performance by a substantial margin on all tasks. It suggests that \alg{} can effectively tackle the distributional shifts and abstract valuable information from distributed behavioral data. Further, we depict the learning curves in \cref{fig:convergence}. \alg{} exhibits the fastest convergence speed on each task (often within 20 communication rounds), highlighting the great communication efficiency of \alg{}. 


\subsubsection{Results under different numbers of agents}

\begin{figure}[H]
    \centering
    \vspace{-1.0em}
    \subfigure{\includegraphics[width=0.49\columnwidth]{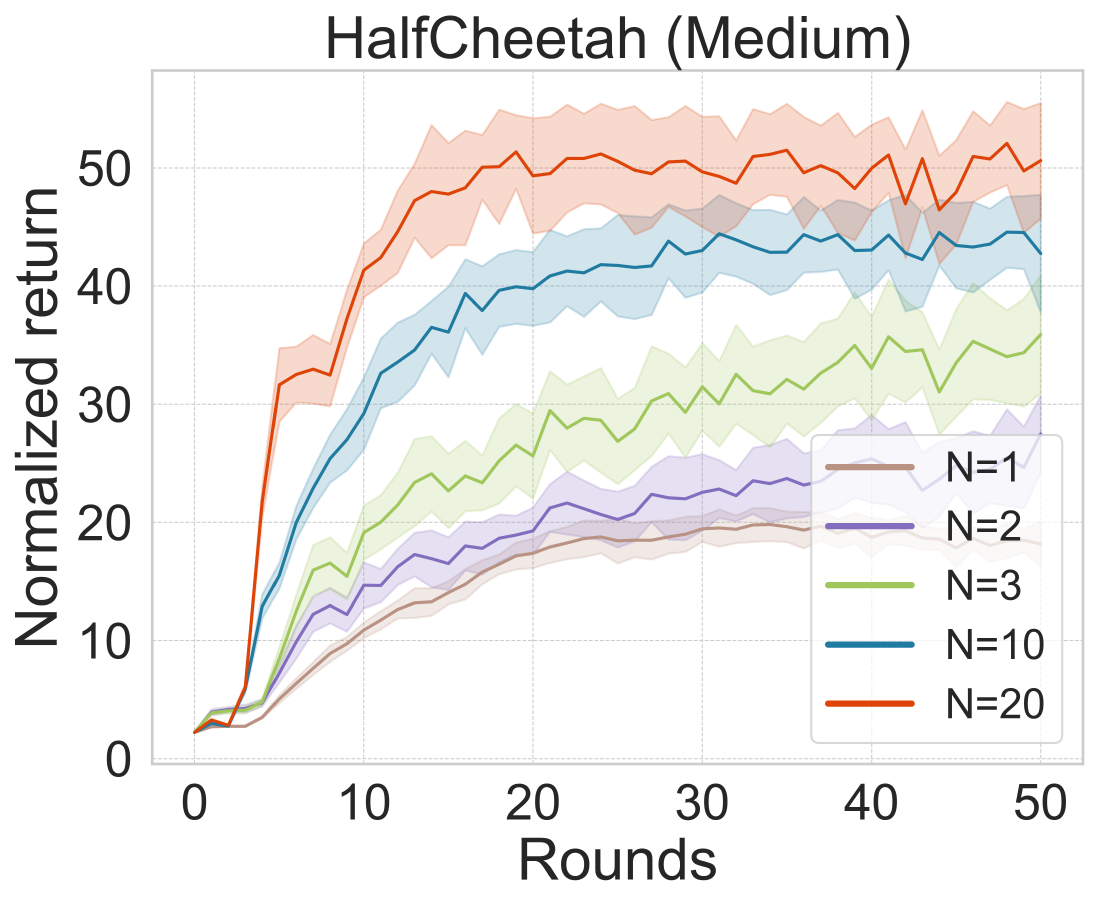}}
    \subfigure{\includegraphics[width=0.49\columnwidth]{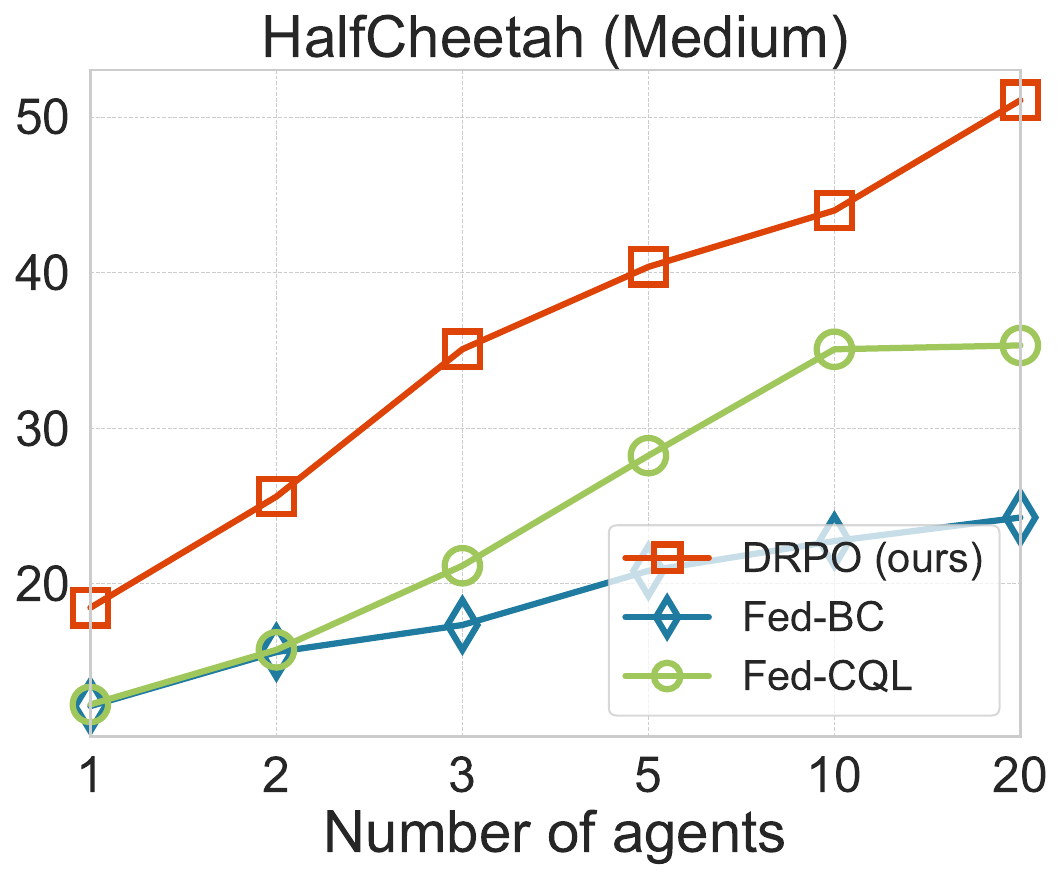}}
    \vspace{-2.0em}
    \caption{Impact of the number of agents on performance.}
    \vspace{-0.5em}
    \label{fig:num_agent}
\end{figure}

In response to the third question, we fix the number of local updating steps to 1k and the local trajectories to 5. We then vary the number of agents from 1 to 20 and present the comparative results in \cref{fig:num_agent}. It indicates that as the number of participating agents increases, the performance of \alg{} exhibits significant improvement, with the gap from baselines widening. Interestingly, the increment in the number of agents also accelerates the convergence of \alg{}, because a great number of  agents can facilitate better policy evaluation, subsequently leading to enhanced policy improvement in each round.

\subsubsection{Results under different dataset sizes}

\begin{figure}[htpb]
    \centering
    \vspace{-1.0em}
    \subfigure{\includegraphics[width=0.49\columnwidth]{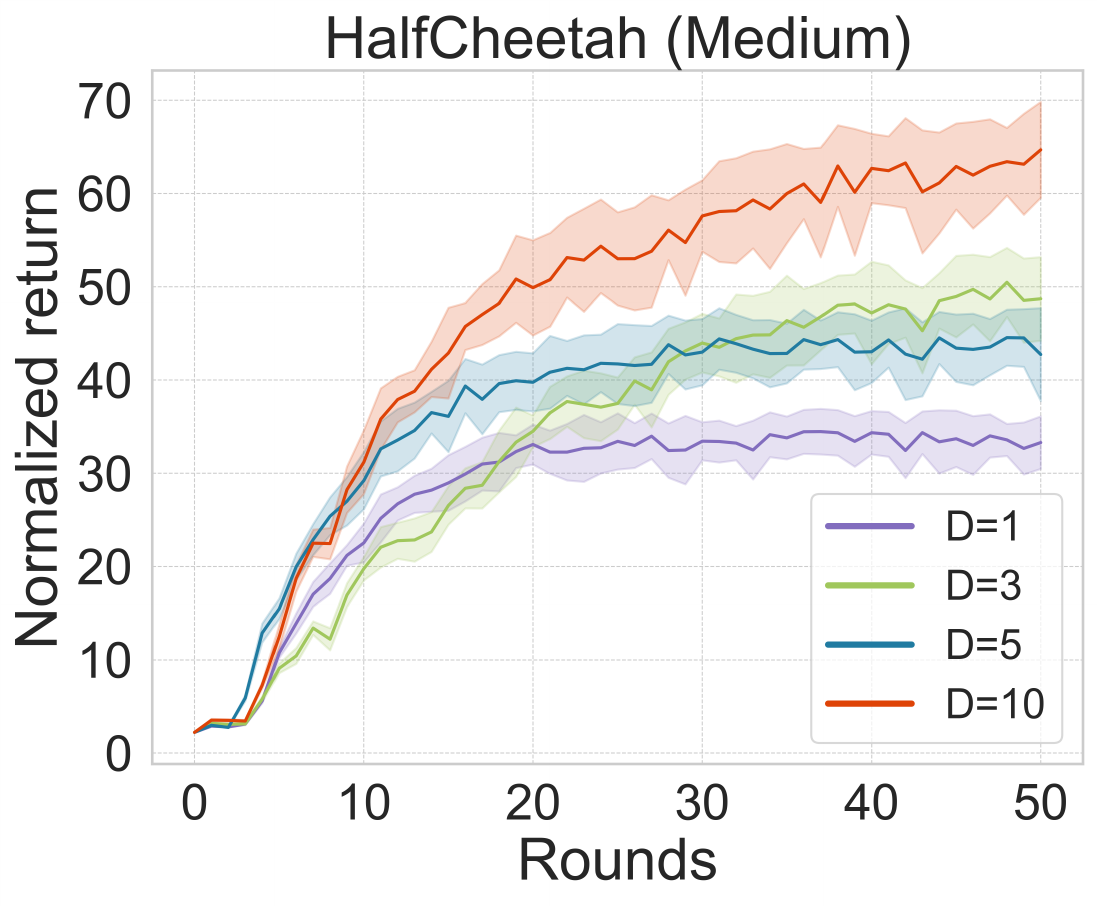}}
    \subfigure{\includegraphics[width=0.49\columnwidth]{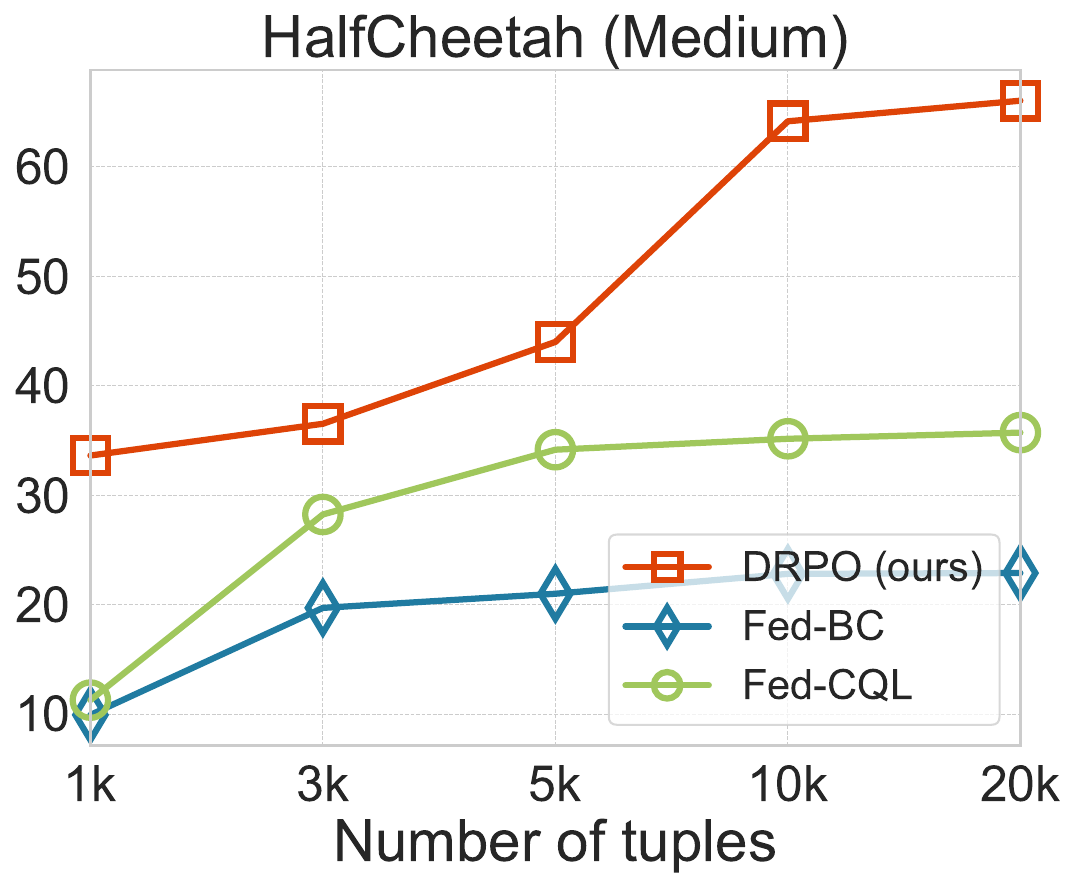}}
    \vspace{-2.0em}
    \caption{Impact of dataset sizes on performance.}
    \vspace{-0.5em}
    \label{fig:datasize}
\end{figure}

To show the impact of dataset sizes on the learning speed, we present results with varying numbers of trajectories in each agent, ranging from 1 to 20, while keeping the number of agents and local steps fixed at 5 and 1k respectively. As shown in \cref{fig:datasize}, the performance of \alg{} increases with more pre-collected data and surpasses the baselines in both small data and large data regimes. 

\subsubsection{Results under varying local steps} 

\begin{figure}[htpb]
    \centering
    \vspace{-1.0em}
    \subfigure{\includegraphics[width=0.49\columnwidth]{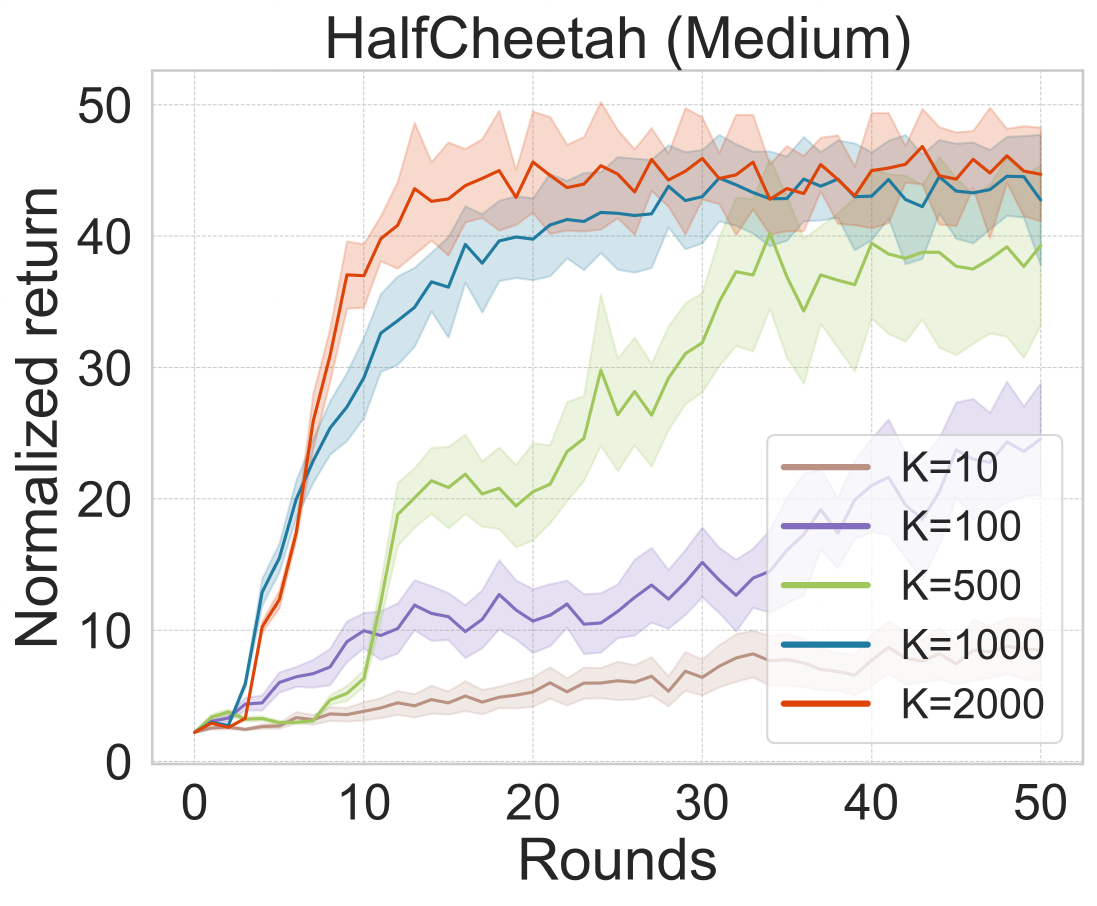}}
    \subfigure{\includegraphics[width=0.49\columnwidth]{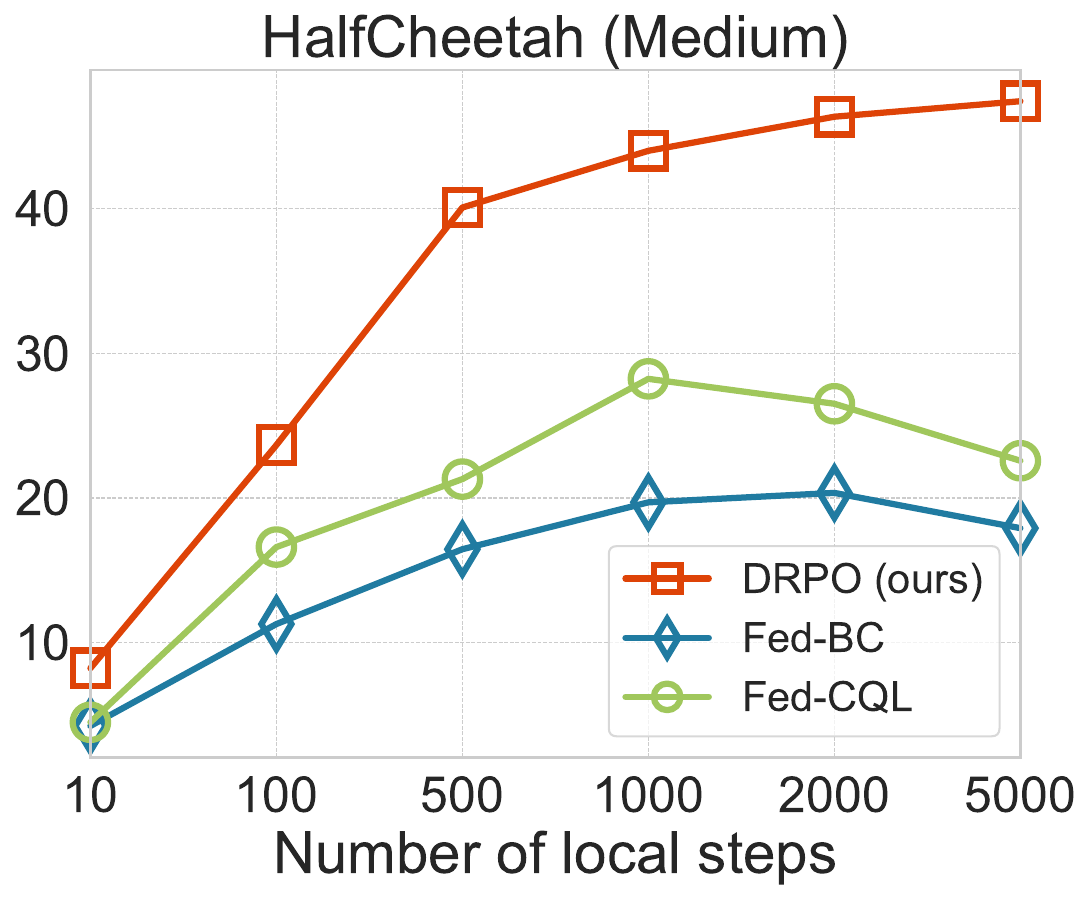}}
    \vspace{-2.0em}
    \caption{Impact of local updating steps on performance.}
    \vspace{-0.5em}
    \label{fig:local}
\end{figure}

\begin{figure*}[t]
    \centering
    \subfigure{
        \begin{minipage}[b]{0.2\textwidth}
            \centering
            \includegraphics[width=1.25\linewidth]{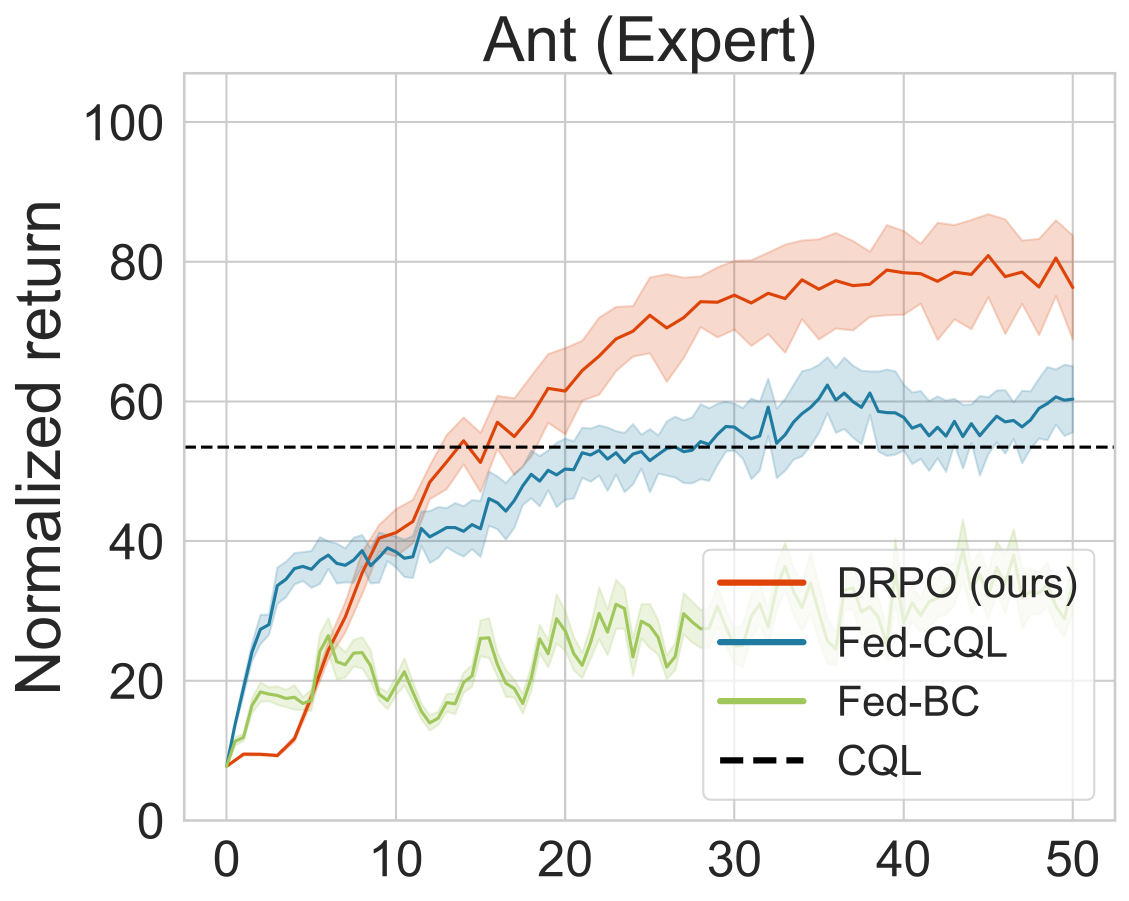}
            \includegraphics[width=1.25\linewidth]{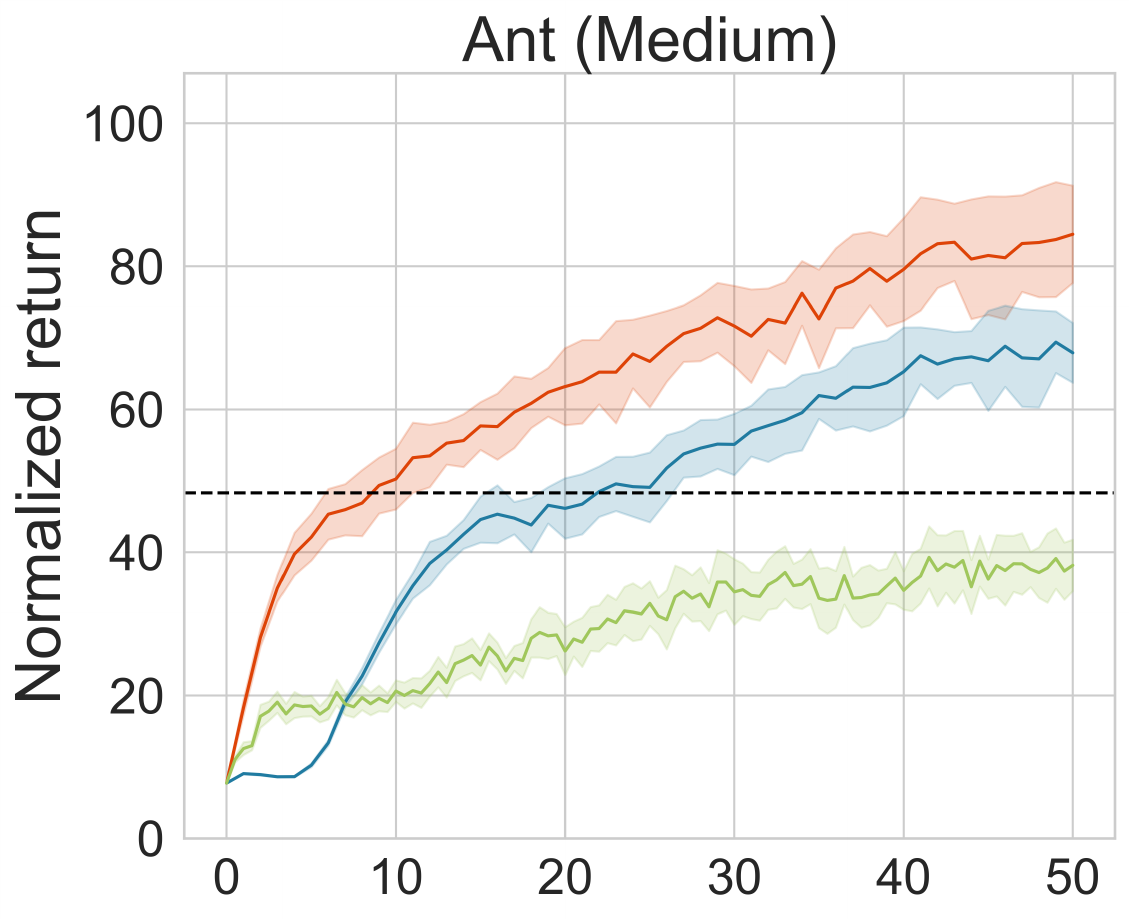}
            \includegraphics[width=1.25\linewidth]{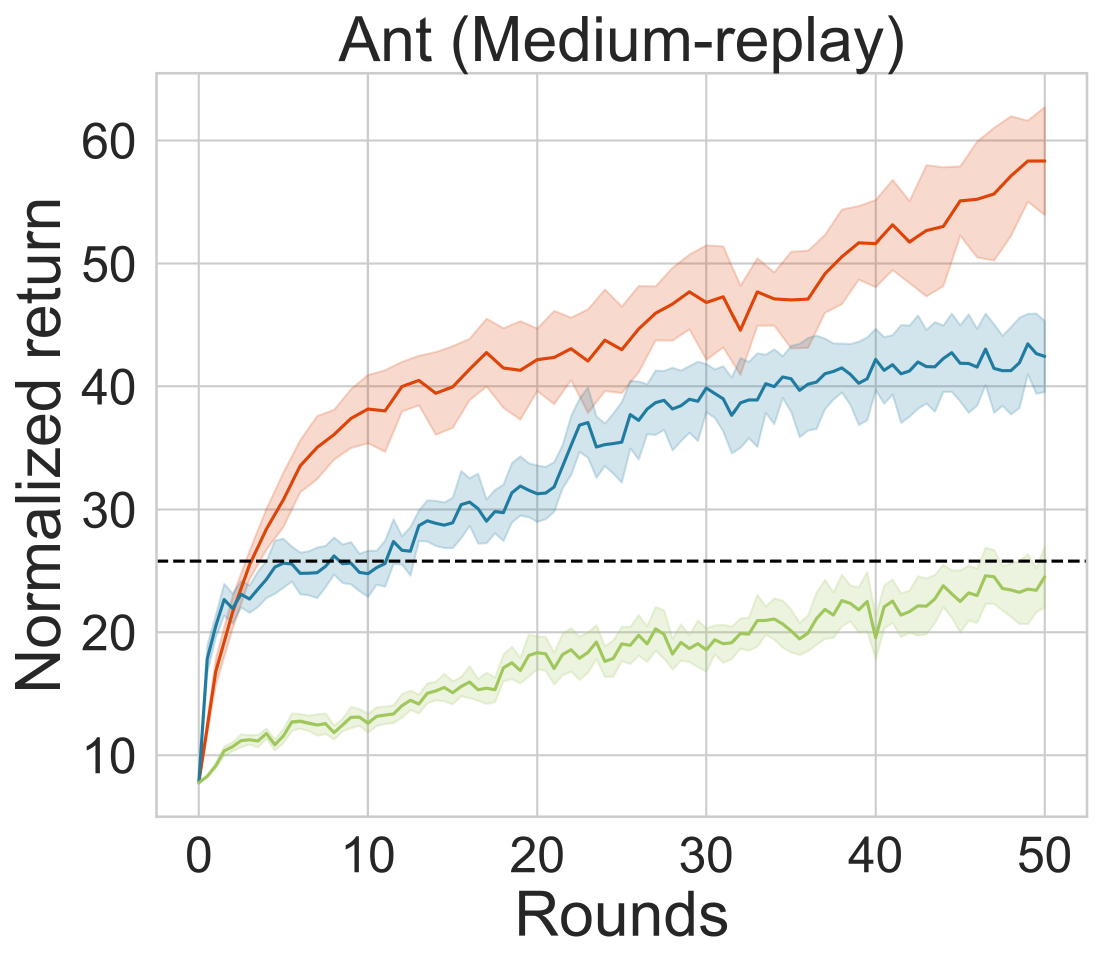}
        \end{minipage}
    }
    \hfill
    \subfigure{
        \begin{minipage}[b]{0.2\textwidth}
            \centering
            \includegraphics[width=1.25\linewidth]{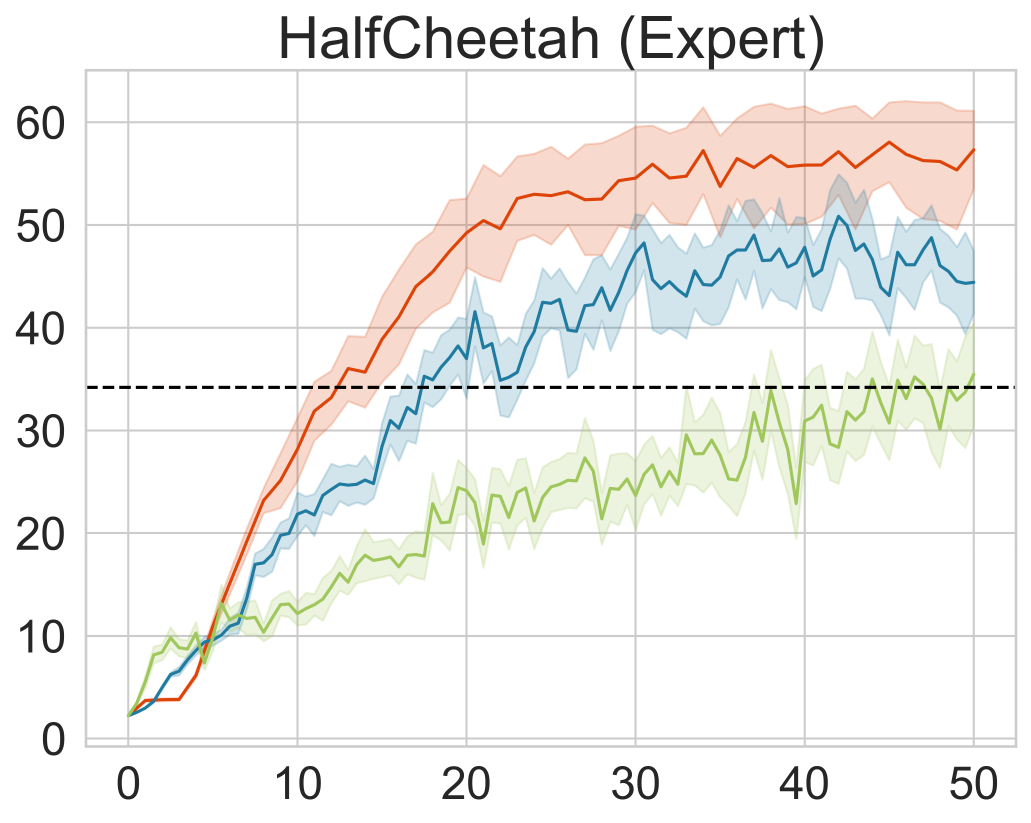}
            \includegraphics[width=1.25\linewidth]{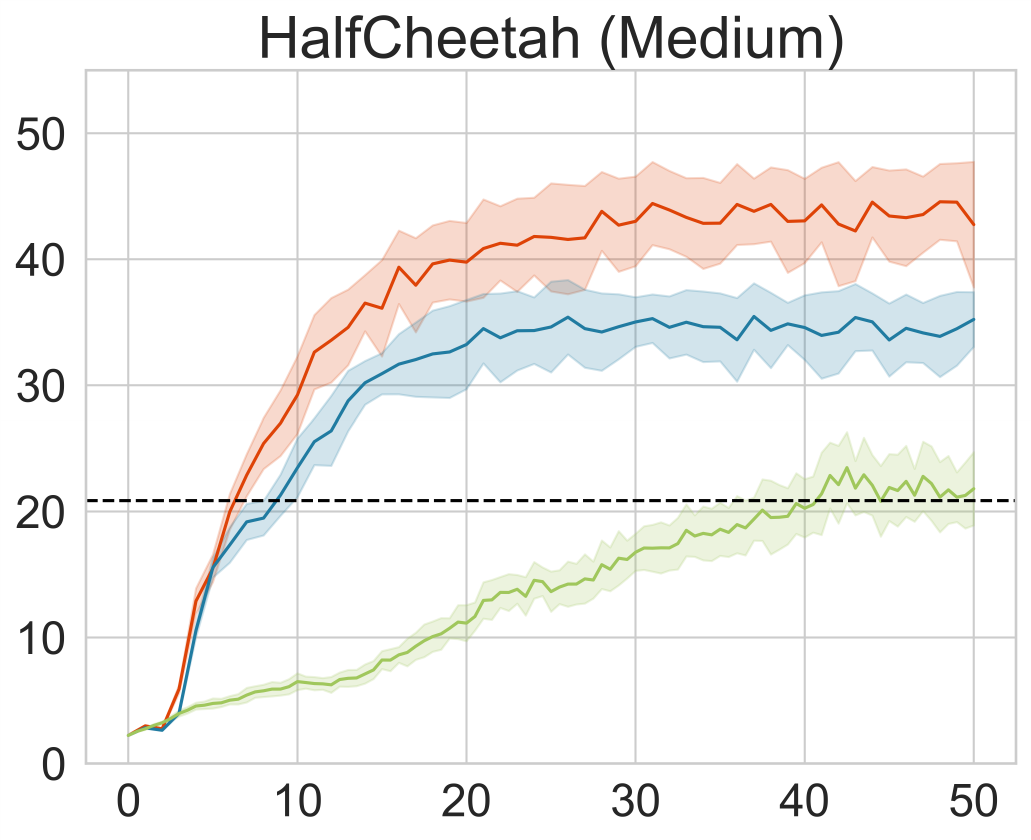}
            \includegraphics[width=1.25\linewidth]{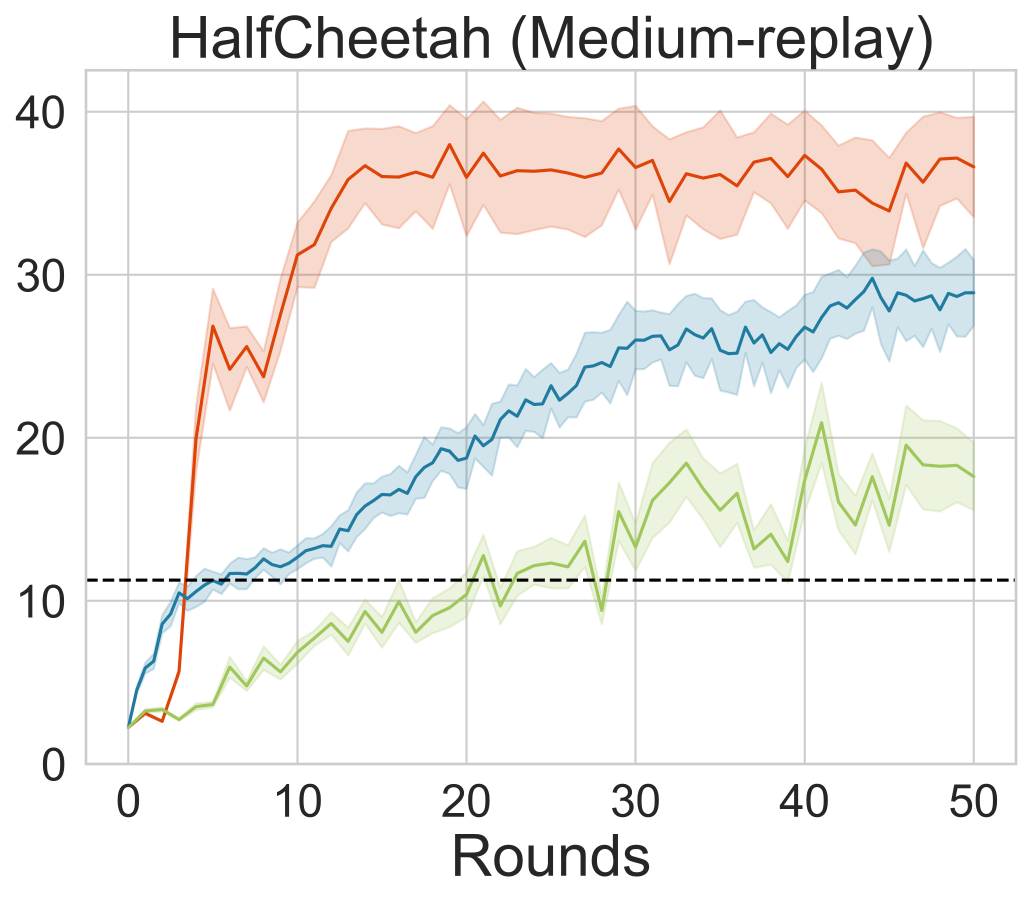}
        \end{minipage}
    }
    \hfill
    \subfigure{
        \begin{minipage}[b]{0.2\textwidth}
            \centering
            \includegraphics[width=1.25\linewidth]{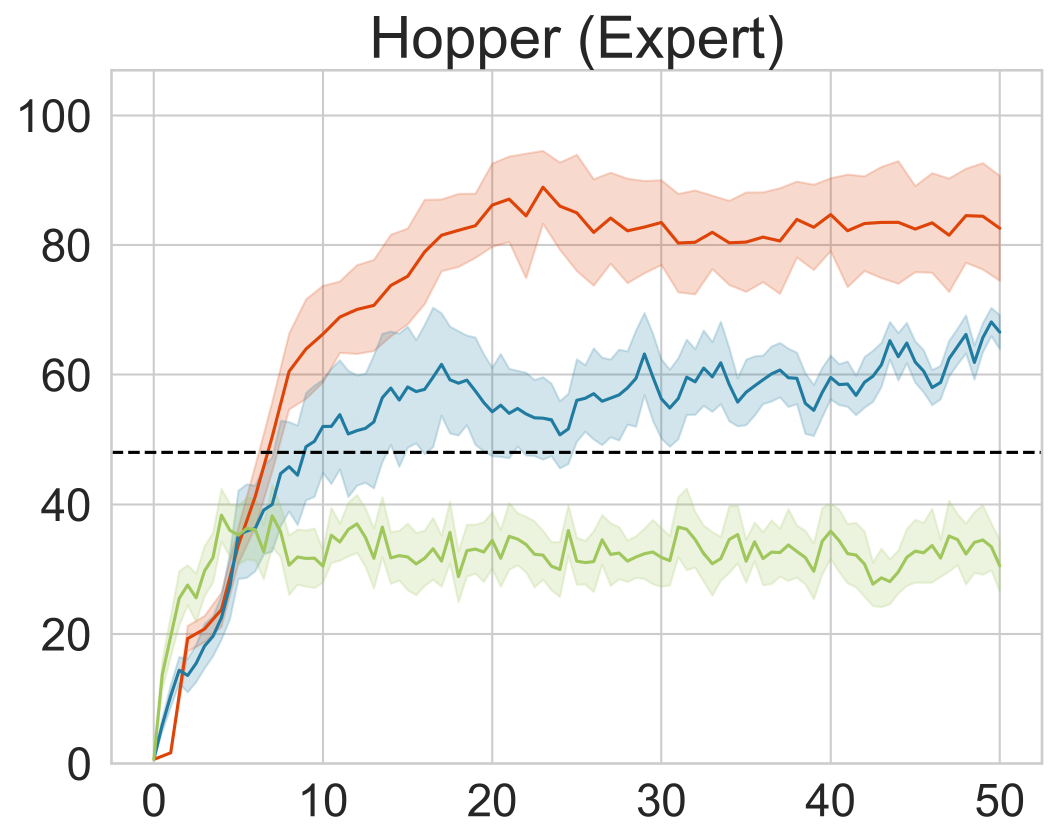}
            \includegraphics[width=1.25\linewidth]{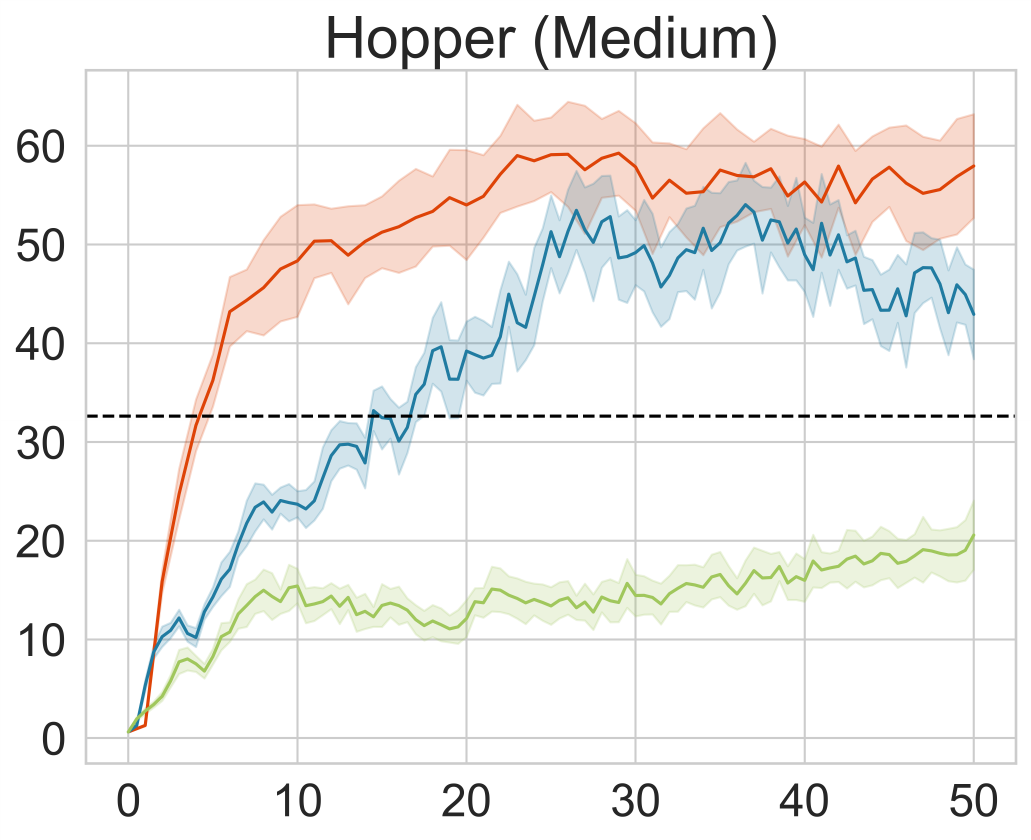}
            \includegraphics[width=1.25\linewidth]{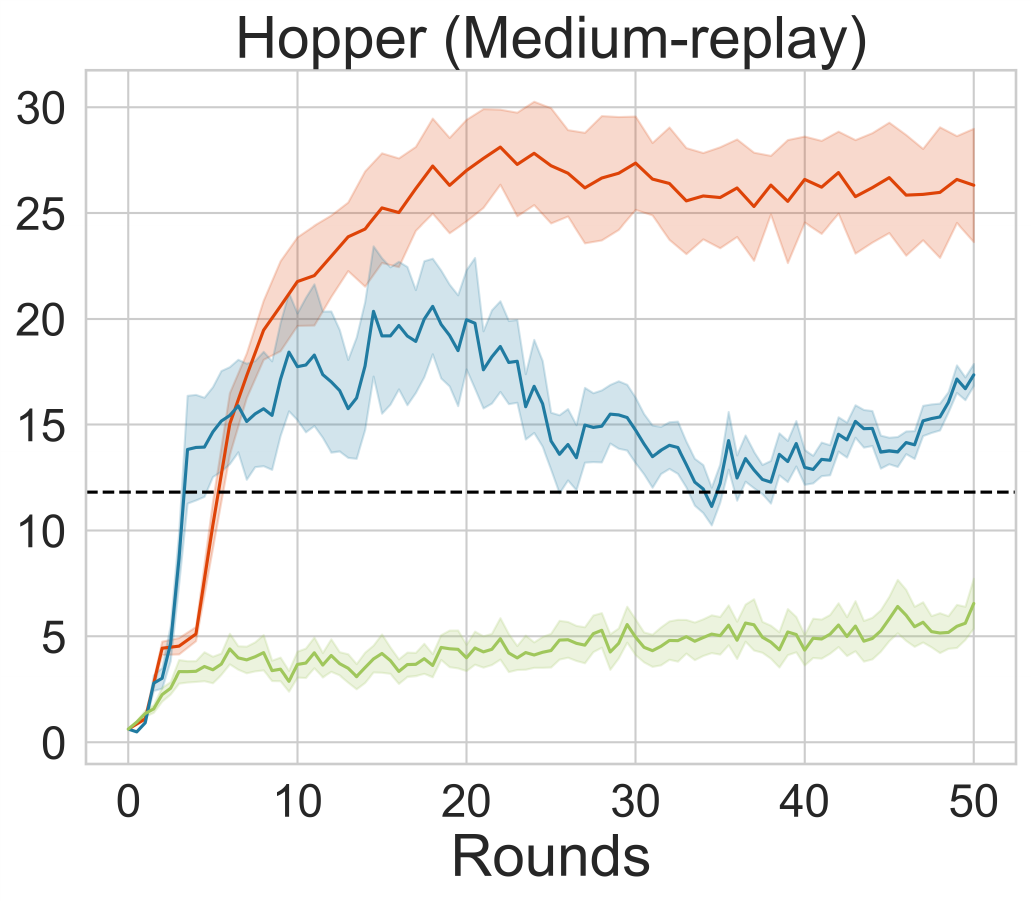}
        \end{minipage}
    }
    \hfill
    \subfigure{
        \begin{minipage}[b]{0.2\textwidth}
            \centering
            \includegraphics[width=1.25\linewidth]{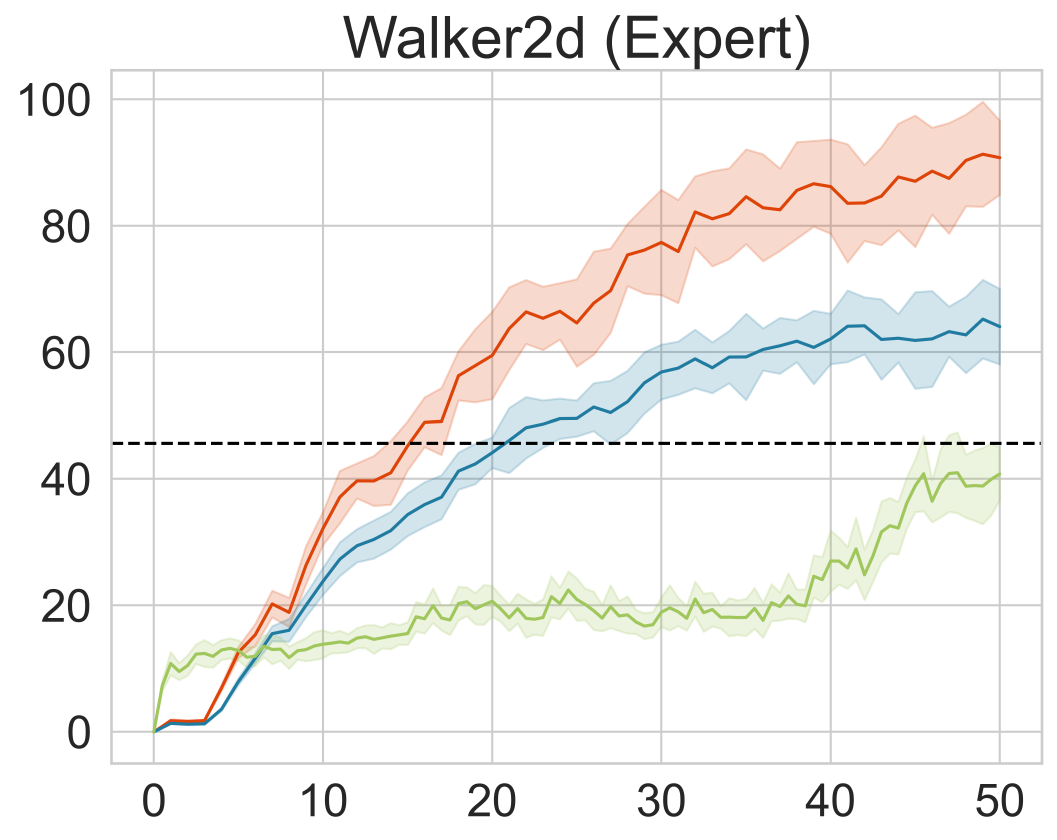}
            \includegraphics[width=1.25\linewidth]{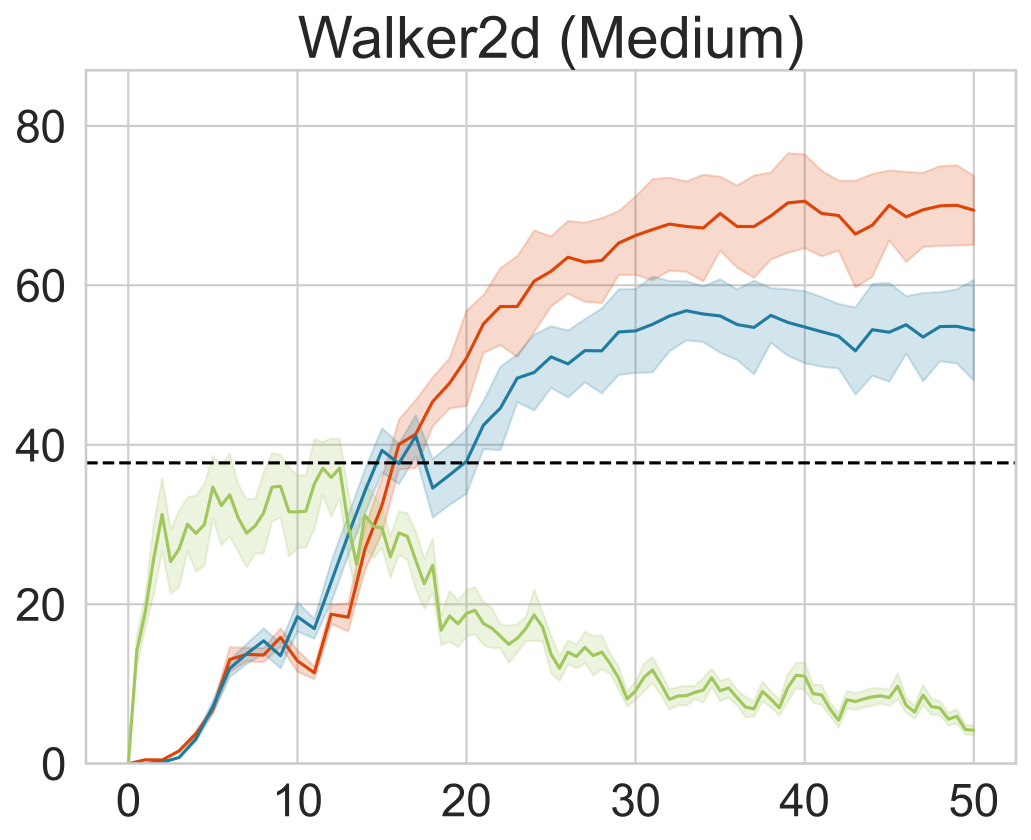}
            \includegraphics[width=1.25\linewidth]{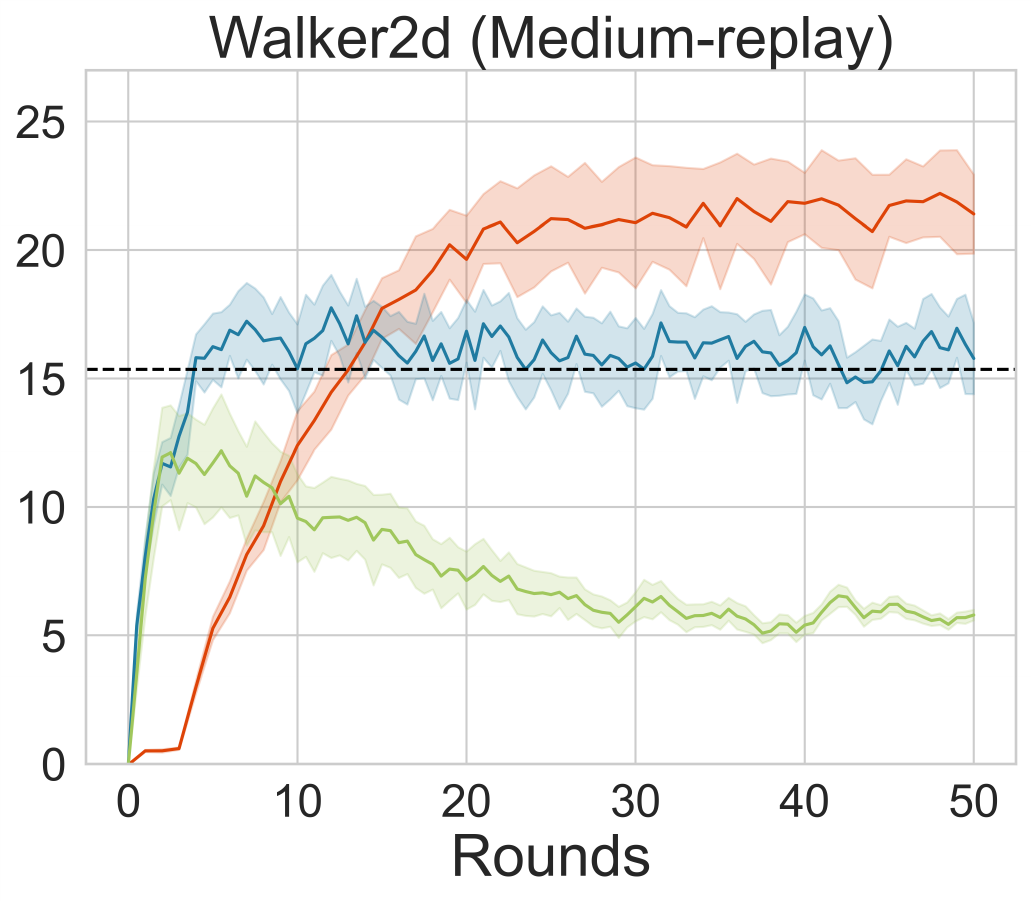}
        \end{minipage}
    }
    \hfill
    \vspace{-.5em}
    \caption{Comparison of learning speeds across different algorithms.}
    \vspace{-1.0em}
    \label{fig:convergence}
\end{figure*}

To demonstrate the impacts of local updates, we provide results with varying numbers of local steps in \cref{fig:local}, while keeping the number of agents and local trajectories fixed at 5.   \cref{fig:local} illustrates that larger local steps lead to a faster convergence speed for \alg{}. In particular, due to overfitting local data, the performances of baselines may deteriorate with an increment in local steps. In contrast, as implied by \cref{fig:local}, \alg{} with dual regularization effectively ameliorates this issue.


\subsubsection{Impact of dual regularization}

\begin{figure}[ht]
    \centering
    \vspace{-0.75em}
    \subfigure[Ablation study.]{\label{fig:ablation}\includegraphics[width=0.49\columnwidth]{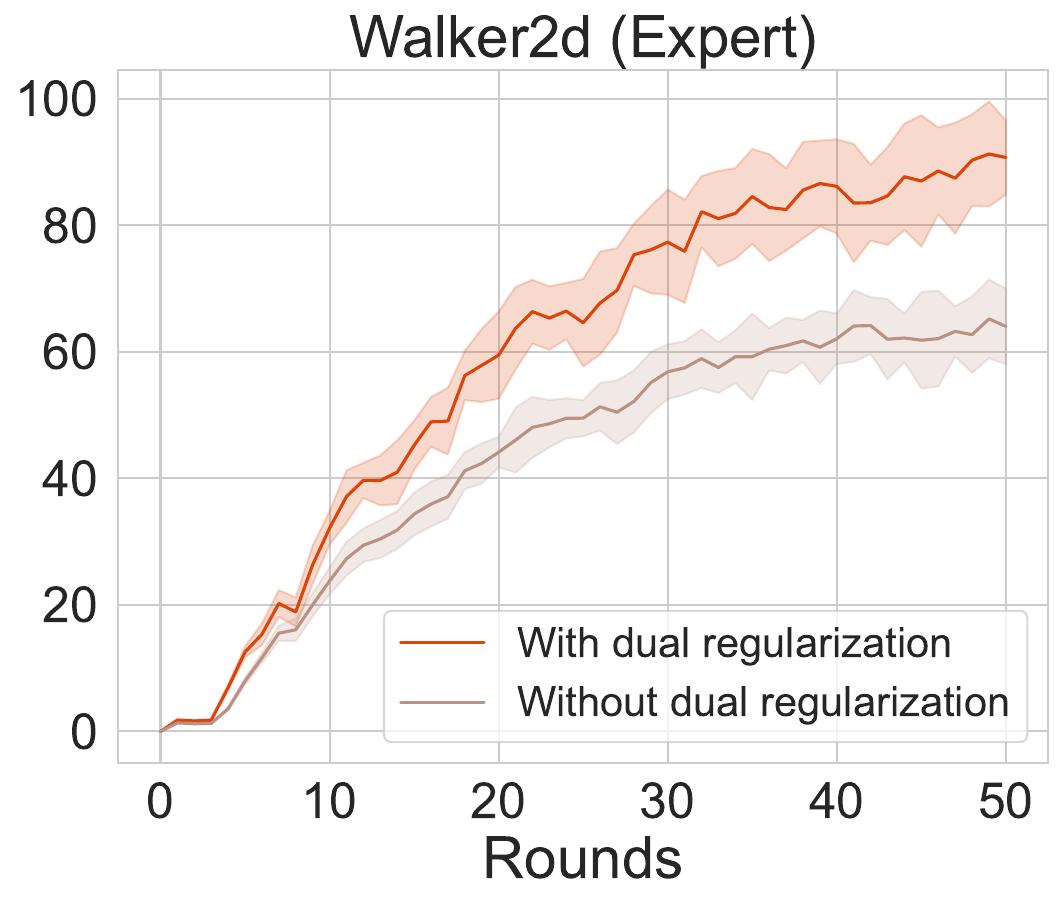}}
    \subfigure[Impacts of $\lambda_1$ and $\lambda_2$.]{\label{fig:lambda}\includegraphics[width=0.49\columnwidth]{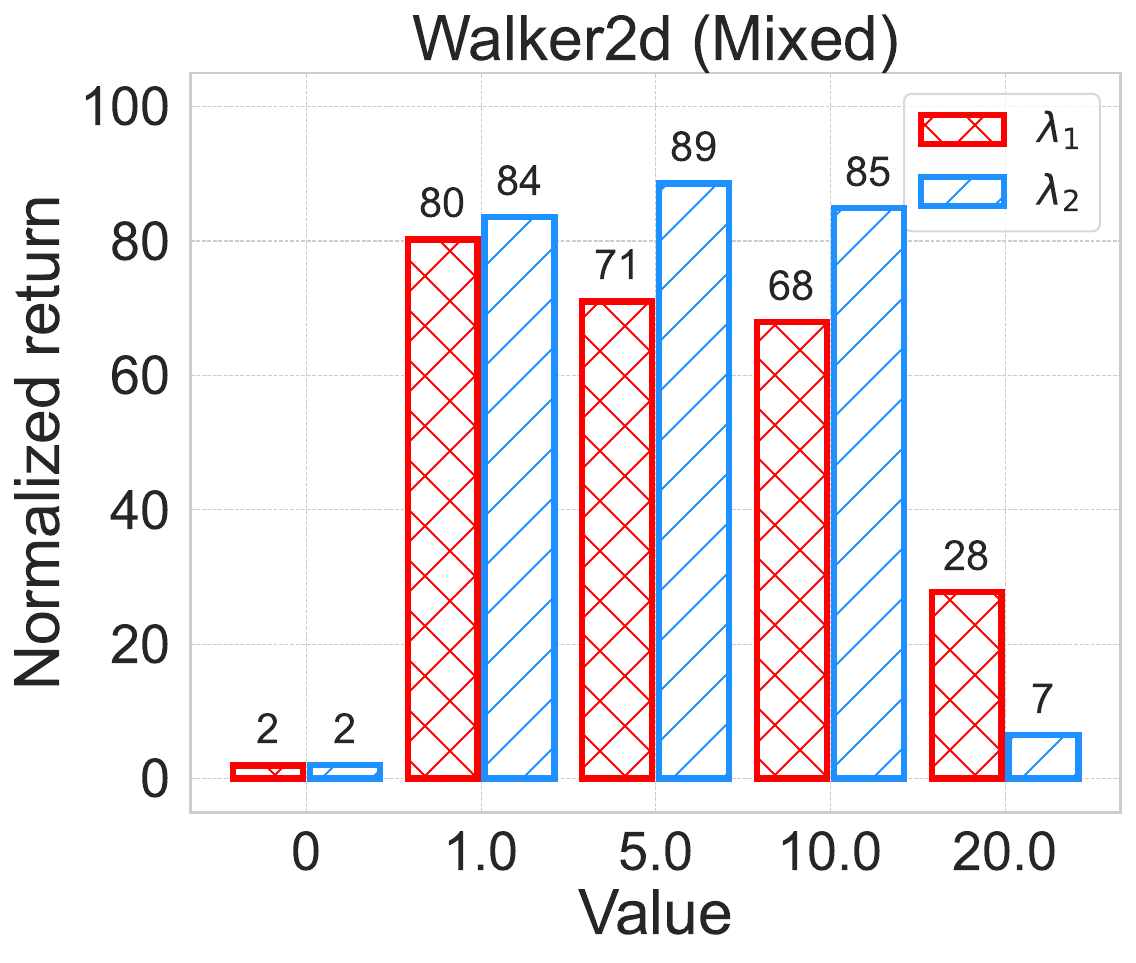}}
    \vspace{-0.75em}
    \caption{Impact of the dual regularization on performance.}
    \vspace{-0.5em}
    \label{fig:dual_regularization}
\end{figure}

To show the effectiveness of dual regularization, we carry out an ablation study by removing the two regularizers (where \alg{} reduces to \texttt{Fed-CQL}). As demonstrated in \cref{fig:ablation}, the regularization significantly contributes to both convergence and performance improvements. Further, to show the individual impacts of the two regularizers, we construct a heterogeneous setting with 4 \texttt{expert} agents and 1 \texttt{medium} agent. We first fix $\lambda_2=1$ and vary $\lambda_1$ from 0 to 20, and then fix $\lambda_1=0.1$ and vary $\lambda_2$ within the same range. As illustrated in \cref{fig:lambda}, setting $\lambda_1$ and $\lambda_2$ to values that are too smaller or too large can result in performance degradation due to the tradeoff between the exploitation of local and global information. Notably, the best result is achieved when $\lambda_1=0.1$ and $\lambda_2=5$, which aligns with our theoretical findings in \cref{thm:policy_improvement}.



\section{Conclusion}
\label{sec:conclusion}
In this paper, we present a novel offline federated policy optimization algorithm, \alg{}, that enables distributed agents to collaboratively learn a decision policy only from pre-collected private data with no need to explore the environments. The key building block of \alg{} lies in its double regularization, one based on the local state-action distribution and the other on the global aggregated policy, which can effectively tackle the inherent two-tier distributional mismatches in offline FRL. We theoretically establish the strict policy improvement guarantee for \alg{}, and empirically demonstrate the efficacy and effectiveness of \alg{} on challenging MuJoCo tasks. The authors have provided public access to the source code at 
\href{https://codeocean.com/capsule/8103588/tree/v1}{https://codeocean.com/capsule/8103588/tree/v1}.


\section*{Acknowledgments}
This research was supported in part by NSFC under Grant No. 62341201, 62122095, 62072472, 62172445, 62302260, and 62202256 by the National Key R\&D Program of China under Grant No. 2022YFF0604502, by CPSF Grant 2023M731956, and by a grant from the Guoqiang Institute, Tsinghua University.

\appendix

We provide necessary lemmas aiding in the analysis of \cref{sec:theory}. First, we bound the difference between state-action distributions w.r.t. the policy distance.
\begin{lemma}
    \label{lem:markov_tvd}
    For two MDP transition distributions $p_1(s'\vert s)$ and $p_2(s'\vert s)$, suppose the total variation distance is bounded as $\max_h \mathbb{E}_{s\sim p_{1,h}}[\TV (p_1(s'\vert s)\Vert p_2(s'\vert s))]\le\epsilon$ with $p_{1,h},p_{2,h}$ the state marginals in time step $h$. If the initial state distributions are the same, then the following fact holds:
    \begin{align}
        \TV (p_{1,h}(s)\Vert p_{2,h}(s))\le h\epsilon.
    \end{align}
\end{lemma}
\begin{proof}
    First we have 
    \begin{align}
        \label{eq:markov_tvd_1}
        &\,\vert p_{1,h}(s)-p_{2,h}(s)\vert\nonumber\\
        =&\,\Big\vert \sum_{s'}p_{1,h-1}(s')\cdot p_1(s\vert s')-p_{2,h-1}(s')\cdot p_2(s\vert s')\Big\vert\nonumber\\
        \le&\, \sum_{s'}\Big\vert p_{1,h-1}(s')\cdot p_1(s\vert s')-p_{2,h-1}(s')\cdot p_2(s\vert s')\Big\vert\nonumber\\
        =&\,\sum_{s'}\Big\vert p_{1,h-1}(s')\cdot p_1(s\vert s')-p_{1,h-1}(s')\cdot p_2(s\vert s')\nonumber\\
        &+p_{1,h-1}(s')\cdot p_2(s\vert s')-p_{2,h-1}(s')\cdot p_2(s\vert s')\Big\vert\nonumber\\
        \le&\,\sum_{s'}\Big(p_{1,h-1}(s')\cdot\big\vert p_1(s\vert s')-p_2(s\vert s')\big\vert\nonumber\\
        &+p_2(s\vert s')\cdot\big\vert p_{1,h-1}(s')-p_{2,h-1}(s')\big\vert\Big)\nonumber\\
        =&\,\mathbb{E}_{s'\sim p_{1,h-1}(s')}\Big[\big\vert p_1(s\vert s')-p_2(s\vert s')\big\vert\Big]\nonumber\\
        &+\sum_{s'}p_2(s\vert s')\cdot\big\vert p_{1,h-1}(s')-p_{2,h-1}(s')\big\vert.
    \end{align}
    We let $\epsilon_{h-1}\doteq\mathbb{E}_{s'\sim p_{1,h-1}(s')}\left[\TV (p_1(s\vert s')\Vert p_2(s\vert s'))\right]$. Then, based on \cref{eq:markov_tvd_1}, the following holds:
    \begin{align}
        &\TV (p_{1,h}(s)\Vert p_{2,h}(s))\nonumber\\
        =&\frac{1}{2}\sum_s \left\vert p_{1,h}(s)- p_{2,h}(s)\right\vert\nonumber\\
        \le& \frac{1}{2}\sum_s \bigg(\mathbb{E}_{s'\sim p_{1,h-1}(s')}\left[\left\vert p_1(s\vert s')-p_2(s\vert s')\right\vert\right]\nonumber\\
        &+\sum_{s'}p_2(s\vert s')\left\vert p_{1,h-1}(s')-p_{2,h-1}(s')\right\vert\bigg)\nonumber\\
        =&\frac{1}{2}\sum_s \mathbb{E}_{s'\sim p_{1,h-1}(s')}\left[\left\vert p_1(s\vert s')-p_2(s\vert s')\right\vert\right]\nonumber\\
        &+\TV (p_{1,h-1}(s')\Vert p_{2,h-1}(s'))\nonumber\\
        =&\delta_{h-1} +\TV (p_{1,h-1}(s')\Vert p_{2,h-1}(s'))\nonumber\\
        \le& \TV (p_{1,0}(s')\Vert p_{2,0}(s')) + \sum^{h-1}_{t=0}\epsilon_{t}\nonumber\\
        \le& h\epsilon,
    \end{align}
    thereby completing the proof.
\end{proof}

Building on \cref{lem:markov_tvd}, we have the following result.
\begin{lemma}
    \label{lem:occupancy_bound}
    For any MDP and policies $\pi,\bar{\pi}$, the corresponding occupancy measures satisfy
    \begin{align}
        \TV(\rho^{\pi}\|\rho^{\bar{\pi}}) \le \frac{1}{1-\gamma} \max_s \TV(\pi(\cdot|s)\|\bar{\pi}(\cdot|s)).
    \end{align}
\end{lemma}
\begin{proof}
    From the definition of occupancy measure, we get
    \begin{align}
        \TV(\rho^{\pi}\|\rho^{\bar{\pi}})\le (1-\gamma)\sum^\infty_{h=0}\gamma^h\TV(\rho^{\pi}_h\|\rho^{\bar{\pi}}_h).
        \label{eq:eq5}
    \end{align}
    Regarding $\TV(\rho^{\pi}_h\|\rho^{\bar{\pi}}_h)$, we can write 
    \begin{align}
        &\TV (\rho^{\pi}_h\|\rho^{\bar{\pi}}_h)\nonumber\\
        =&\frac{1}{2}\sum_{s,a}\left\vert \rho^{\pi}_h(s,a)- \rho^{\bar{\pi}}_h(s,a)\right\vert\nonumber\\
        =&\frac{1}{2}\sum_{s,a}\Big\vert \rho^{\pi}_h(s)\pi(a\vert s)-\rho^{\pi}_h(s)\bar{\pi}(a\vert s)\nonumber\\
        &+\rho^{\pi}_h(s)\bar{\pi}(a\vert s)- \rho^{\bar{\pi}}_h(s)\bar{\pi}(a\vert s)\Big\vert\nonumber\\
        \le&\frac{1}{2}\sum_{s,a}\rho^{\pi}_h(s)\left\vert \pi(a\vert s)-\bar{\pi}(a\vert s)\right\vert\nonumber\\
        &+\frac{1}{2}\sum_{s,a}\bar{\pi}(a\vert s)\left\vert \rho^{\pi}_h(s)- \rho^{\bar{\pi}}_h(s)\right\vert\nonumber\\
        =&\sum_s \rho^{\pi}_h(s)\cdot\frac{1}{2}\sum_a\left\vert \pi(a\vert s)-\bar{\pi}(a\vert s)\right\vert\nonumber\\
        &+\frac{1}{2}\sum_s \left\vert \rho^{\pi}_h(s)- \rho^{\bar{\pi}}_h(s)\right\vert\sum_a \bar{\pi}(a\vert s)\nonumber\\
        =&\mathbb{E}_{s\sim \rho^{\pi}_h}\left[\TV(\pi(\cdot\vert s)\Vert \bar{\pi}(\cdot\vert s))\right]+\TV(\rho^{\pi}_h(s)\Vert \rho^{\bar{\pi}}_h(s))\nonumber\\
        \le&\max_s \TV(\pi(\cdot\vert s)\Vert \bar{\pi}(\cdot\vert s))+\TV(\rho^{\pi}_h(s)\Vert \rho^{\bar{\pi}}_h(s)),
        \label{eq:eq4}
    \end{align}
    where we sightly abuse notation in $\TV(\rho^{\pi}_h(s)\Vert \rho^{\bar{\pi}}_h(s))$ and denote $\rho^{\pi}_h(s)$ as the state marginal of $\rho^{\pi}_h$. Note that
    \begin{align}
        &\max_h \mathbb{E}_{s\sim\rho^{\pi}_h(s)}\big[\TV (\dy^\pi(\cdot\vert s)\Vert \dy^{\bar{\pi}}(\cdot\vert s))\big]\nonumber\\
        \le& \max_s \TV (\dy^\pi(\cdot\vert s)\Vert \dy^{\bar{\pi}}(\cdot\vert s))\nonumber\\
        =& \max_s \frac{1}{2}\sum_{s'} \bigg|\sum_a P(s'|s,a)\Big(\pi(a|s)-\bar{\pi}(a|s)\Big)\bigg|\nonumber\\
        \le& \max_s \frac{1}{2}\sum_a \big|\pi(a|s)-\bar{\pi}(a|s)\big| \sum_{s'}P(s'|s,a)\nonumber\\
        =& \max_s \TV(\pi(\cdot|s)\|\bar{\pi}(\cdot|s)).
    \end{align}
    Hence, based on \cref{lem:markov_tvd,eq:eq4} , we can write
    \begin{align}
        \TV (\rho^{\pi}_h\|\rho^{\bar{\pi}}_h) \le (h+1)\max_s \TV(\pi(\cdot\vert s)\Vert \bar{\pi}(\cdot\vert s)).
        \label{eq:eq6}
    \end{align}
    Finally, combining \cref{eq:eq5,eq:eq6}, we obtain
    \begin{align}
        \TV(\rho^{\pi}\|\rho^{\bar{\pi}})&\le \frac{1-\gamma}{\gamma}\sum^\infty_{h=1}h\gamma^h \max_s \TV(\pi(\cdot\vert s)\Vert \bar{\pi}(\cdot\vert s))\nonumber\\
        &=\frac{1}{1-\gamma}\max_s \TV(\pi(\cdot\vert s)\Vert \bar{\pi}(\cdot\vert s)),
    \end{align}
    with the fact $\gamma<1$ and $\sum^\infty_{h=1}h\gamma^h=\gamma/(1-\gamma)^2$. 
\end{proof}


\bibliographystyle{IEEEtran}
\bibliography{reference}

\end{document}